\documentstyle[jair,twoside,11pt,theapa,amssymb,graphicx,epsfig]{article}
\input epsf

\newtheorem{theorem}{Theorem}[section]

\newcommand\ratio{\frac{\lambda}{(i+1)\mu}} 
\newcommand\ratioSimple{\frac{\lambda}{\mu}}
\newcommand\diff{k_{i+1}-k_i}

\newcommand\lambdamuiRatio{\frac{\lambda}{i\mu}}
\newcommand\iRatio{\frac{1}{i}}
\newcommand\betajExpression{\left({\ratioSimple}\right)^{j-k_0}{\left(\iRatio\right)}^{j-k_{i-1}}X_i } 

\newenvironment{proof}
{\begin{flushleft} \begin{description}
\item \textit{\textbf{Proof:}}~ } 
{\hfill{$\Box$} 
\end{description}\end{flushleft} }  

\jairheading{32}{2008}{123--167}{09/07}{05/08}
\ShortHeadings{A CP Approach for a Queueing Control Problem}
{Terekhov \& Beck}

\firstpageno{123}

\begin{document}

\title{A Constraint Programming Approach for Solving a Queueing Control Problem}

\author{\name Daria Terekhov \email dterekho@mie.utoronto.ca \\ \name J. Christopher Beck \email jcb@mie.utoronto.ca \\
    \addr Department of Mechanical \& Industrial Engineering \\
    University of Toronto, Canada}

\maketitle

\begin{abstract}
%

In a facility with front room and back room operations, it is useful to
switch workers between the rooms in order to cope with changing customer
demand. Assuming stochastic customer arrival and service times, we seek a
policy for switching workers such that the expected customer waiting time is
minimized while the expected back room staffing is sufficient to perform all
work.  Three novel constraint programming models and several shaving
procedures for these models are presented. Experimental results show that a
model based on closed-form expressions together with a combination of
shaving procedures is the most efficient. This model is able to find and
prove optimal solutions for many problem instances within a reasonable
run-time. Previously, the only available approach was a heuristic
algorithm. Furthermore, a hybrid method combining the heuristic and the best
constraint programming method is shown to perform as well as the heuristic
in terms of solution quality over time, while achieving the same performance
in terms of proving optimality as the pure constraint programming
model. This is the first work of which we are aware that solves such
queueing-based problems with constraint programming.

\end{abstract}

\section{Introduction}
\label{sec:introduction}

The original motivation for the study of scheduling and resource allocation
problems within artificial intelligence (AI) and constraint programming (CP)
was that, in contrast to Operations Research (OR), the full richness of the
problem domain could be represented and reasoned about using techniques from
knowledge representation \cite{Fox83}. While much of the success of
constraint-based scheduling has been due to algorithmic advances
\cite{Baptiste01a}, recently, there has been interest in more
complex problems such as those involving uncertainty
\cite{Policella04a,Sutton07a,Beck07a}. In the broader constraint
programming community there has been significant work over the past five
years on reasoning under uncertainty \cite{Brown06}. Nonetheless, it is
recognized that ``constraint solving under change and uncertainty is in its
infancy'' \cite[p. 754]{Brown06}.


Queueing theory has intensively studied the design and control of systems
for resource allocation under uncertainty \cite{Gross98a}. Although much of the
study has been descriptive in the sense of developing mathematical models of
queues, there is prescriptive work that attempts to develop queue designs
and control policies so as to optimize quantities of interest (e.g., a
customer's expected waiting time) \cite{Tadj05a}. One of the challenges to queueing theory,
however, is that analytical models do not yet extend to richer
characteristics encountered in real-world problems such as rostering for
call centres \cite{Cezik08a}. 

Our long-term goal is to integrate constraint programming and queueing
theory with two ends in mind: the extension of constraint programming to
reason better about uncertainty and the expansion of the richness of the
problems for which queueing theory can be brought to bear. We do not achieve
this goal here. Rather, this paper represents a first step where we solve a
queueing control problem with constraint programming
techniques. Specifically, we develop a constraint programming approach for a
queueing control problem which arises in retail facilities, such as stores
or banks, which have back room and front room operations. In the front room,
workers have to serve arriving customers, and customers form a queue and
wait to be served when all workers are busy. In the back room, work does not
directly depend on customer arrivals and may include such tasks as sorting
or processing paperwork. All workers in the facility are cross-trained and
are assumed to be able to perform back room tasks equally well and serve
customers with the same service rate.  Therefore, it makes sense for the
managers of the facility to switch workers between the front room and the
back room depending both on the number of customers in the front room and
the amount of work that has to be performed in the back room. These managers
are thus interested in finding a switching policy that minimizes the
expected customer waiting time in the front room, subject to the constraint
that the expected number of workers in the back room is sufficient to
complete all required work. This queueing control problem has been studied
in detail by Berman, Wang and Sapna
\citeyear{Berman05a}, who propose a heuristic for solving it.

Our contributions are twofold. Firstly, constraint programming is, for the
first time, used to solve a stochastic queueing control problem. Secondly, a
complete approach for a problem for which only a heuristic algorithm existed
previously is presented.

The paper is organized as follows. Section \ref{sec:background} presents a
description of the problem and the work done by Berman et
al. \citeyear{Berman05a}. 
In the next section, three CP models for this problem are proposed. Sections \ref{sec:dominanceRules} and 
\ref{sec:shaving} present methods for improving the efficiency of these models, focusing on dominance rules and shaving procedures, respectively. Section \ref{sec:expResults} shows experimental results comparing the proposed CP models and combinations of inference methods. The performance of the CP techniques is contrasted with that of the heuristic method of Berman et al. Based on these results, a hybrid method is proposed and evaluated in Section \ref{sec:hybrid}. In Section \ref{sec:discussion}, a discussion of the results is presented. Section \ref{sec:related} describes related problems
and states some directions for future work. Section \ref{sec:conclusions} concludes the paper. An appendix containing the derivations of some expressions used in the paper is also included. 
%
%
%

%
\section{Problem Description}
\label{sec:background}
Let $N$ denote the number of workers in the facility, and let $S$ be the
maximum number of customers allowed in the front room at any one
time.\footnote{The notation used by Berman et al. \citeyear{Berman05a} is
adopted throughout this paper.} When there are $S$
customers present, arriving customers will be not be allowed to join the front
room queue and will leave without service. Customers arrive according to a Poisson process with rate
$\lambda$. Service times in the front room follow an exponential
distribution with rate $\mu$. The minimum expected number of workers that is
required to be present in the back room in order to complete all of the
necessary work is assumed to be known, and is denoted by $B_l$, where $l$
stands for `lower bound'. Only one worker is allowed to be switched at a
time, and both switching time and switching cost are assumed to be
negligible. The goal of the problem is to find an optimal approach to
switching workers between the front room and the back room so as to minimize
the expected customer waiting time, denoted $W_q$, while at the same time
ensuring that the expected number of workers in the back room is at least
$B_l$. Thus, a policy needs to be constructed that specifies how many
workers should be in the front room and back room at a particular time and
when switches should occur.
\subsection{Policy Definition \label{sec:policyDefinition}}
The term \emph{policy} is used in the queueing control literature to
describe a rule which prescribes, given a particular queue state, the
actions that should be taken in order to control the queue. Most of the
research on the optimal control of queues has focused on determining when a
particular type of policy is optimal, rather than on finding the actual
optimal values of the parameters of the policy \cite{Gross98a}. The term
\emph{optimal policy} is used in the literature to mean both the optimal
\emph{type of policy} and the optimal \emph{parameter values} for a given
policy type. The distinction between the two is important since showing that
a particular policy type is optimal is a theoretical question, whereas
finding the optimal values for a specific policy type is a computational
one.

The policy type adopted here is the one proposed by Berman et
al. \citeyear{Berman05a}. A policy is defined in terms of quantities $k_i$,
for $i = 0, \dots, N$ and states that there should be $i$ workers in the
front room whenever there are between $k_{i-1}+1$ and $k_i$ customers
(inclusive) in the front room, for $i = 1, 2, \dots, N$. As a consequence of
this interpretation, the following constraints have to hold: $k_i - k_{i-1}
\ge 1$, $k_0 \ge 0$ and $k_N = S$. For example, consider a facility with $S
= 6$ and $N = 3$, and suppose the policy $(k_0, k_1, k_2, k_3) = (0, 2, 3,
6)$ is employed. This policy states that when there are $k_0+1 = 1$ or $k_1
= 2$ customers in the front room, there is one worker in the front room;
when there are $3$ customers, there are $2$ workers; and when there are 4,
5, or 6 customers, all 3 workers are employed in the front. Alternatively,
$k_i$ can be interpreted as an upper bound on the number of customers that
will be served by $i$ workers under the given policy. Yet another
interpretation of this type of switching policy comes from noticing that as
soon as the number of customers in the front room is increased by $1$ from
some particular switching point $k_i$, the number of workers in the front
room changes to $i+1$.  This definition of a policy forms the basis of the
model proposed by Berman et al., with the switching points $k_i$, $i = 0,
\dots, N-1$, being the decision variables of the problem, and $k_N$ being
fixed to $S$, the capacity of the system.
 
This policy formulation does not allow a worker to be permanently assigned
to the back room: the definition of the $k_i$'s is such that every worker
will, in some system state, be serving customers. Due to this
definition, there exist problem instances which are infeasible with this
policy type yet are feasible in reality. Consequently, the proposed policy
formulation is sub-optimal \cite{Terekhov07c}.
However, the goal of the current work is to demonstrate the applicability of
constraint programming to computing the optimal values for the given policy
type and not to address theoretical optimality questions. Therefore, the
term \emph{optimal policy} is used
throughout the paper to refer to the optimal numerical parameters for the policy type proposed by Berman et al. \citeyear{Berman05a}.
%

\subsection{Berman et al. Model \label{sec:bermanetalmodel}}
In order to determine the expected waiting time and the expected number of
workers in the back room given a policy defined by particular values of
$k_i$, Berman et al. first define a set of probabilities, $P(j)$\label{Pj}, for $j =
k_0, k_0+1, \dots, S$. Each $P(j)$ denotes the steady-state (long-run) probability of the queue being in state $j$, that is, of 
there being exactly $j$ customers in the facility. 
Based on the Markovian properties of this queueing system (exponential inter-arrival and service times), Berman et al. define a set of detailed balance equations for the determination of these probabilities:  
%
%
\begin{eqnarray}
P(j) \lambda &=& P(j+1) 1 \mu \qquad \; \; j = k_0, k_0 + 1, \dots , k_1 - 1 \nonumber \\
P(j) \lambda &=& P(j+1) 2 \mu \qquad \; \; j = k_1, k_1 + 1, \dots , k_2 - 1 \nonumber \\
{\vdots} & {} & {\qquad \; \; \vdots \qquad \qquad \qquad \; \vdots} \label{eqn:balance1} \\
P(j) \lambda &=& P(j+1) i \mu \qquad \; \; j = k_{i-1}, k_{i-1} + 1, \dots , k_i - 1 \nonumber \\
{\vdots} & {} & {\qquad \;\; \vdots \qquad \qquad \qquad \; \vdots} \nonumber \\
P(j) \lambda &=& P(j+1) N \mu \qquad j = k_{N-1}, k_{N-1} + 1, \dots , k_N - 1. \nonumber
\end{eqnarray}
The probabilities $P(j)$ also have to satisfy the equation
${\sum_{j={k_0}}^{S} P(j) = 1}$. Intuitively, in steady-state, the average
flow from state $j$ to state $j+1$ has to be equal to the average flow from
state $j+1$ to state $j$ \cite{Gross98a}. Since $P(j)$ can be viewed as the
long-run proportion of time when the system is in state $j$, the mean flow
from state $j$ to $j+1$ is $P(j) \lambda$ and the mean flow from state $j+1$
to $j$ is $P(j+1) i \mu$ for some $i$ between 1 and $N$ depending on the
values of the switching points.

From these equations, Berman et al. derive the following expressions for each $P(j)$:
\begin{eqnarray}
P(j) &=& \beta_jP(k_0), \label{eqn:calculatingPj}
\end{eqnarray}
where\label{betaj}
\begin{eqnarray}
\beta_j &=& \left\{ \begin{array}{rlll}
1 & \mbox{if} & j=k_0 & {} \\
{}&{}&{} \label{eqn:betaj} \\
\betajExpression & \mbox{if} & k_{i-1}+1 \le j \le k_i & i = 1, \dots, N
\end{array}\right., \\
{}&{}&{} \nonumber \\ \label{Xi1}
X_i &=& \prod_{g=1}^{i-1}{{\left( \frac{1}{g} \right)}^{k_g-k_{g-1}}} \label{eqn:Xi1} \\
{} & {} & (X_1 \equiv 1),  i = 1, \dots, N.  \nonumber
\end{eqnarray}
$P(k_0)$ can be calculated using the following equation, which is derived by summing both sides of 
Equation (\ref{eqn:calculatingPj}) over all values of $j$:
\begin{eqnarray}
P(k_0)\sum_{j=0}^{S}\beta_j = 1.
\end{eqnarray}

All quantities of interest can be expressed in terms
of the probabilities $P(j)$. Expected number of workers in the front room\label{F} is
\begin{eqnarray}
F &=& \sum_{i=1}^N  {\sum_{j={k_{i-1}+1}}^{k_i} iP(j)}, \label{eqn:Fexpression1} 
\end{eqnarray}
while the expected number of workers in the back room\label{B1} is
\begin{eqnarray}
B &=& N - F \label{eqn:Bexpression}.
\end{eqnarray}
The expected number of customers\label{L} in the front room is 
\begin{eqnarray}
L &=& \sum_{j={k_0}}^{S} jP(j). \label{eqn:Lexpression1}
\end{eqnarray}
Expected waiting time in the queue can be expressed as
\begin{eqnarray}
W_q &=& \frac{L}{\lambda (1 - P(k_N))} - \frac{1}{\mu}. \label{eqn:Wqexpression}
\end{eqnarray}
This expression is derived using Little's Laws for a system of capacity $k_N
= S$. 

Given a family of switching policies\label{Kbf}\label{Kpolicy1} $\textbf{K} = \{ K; K =
\{k_0, k_1, ..., k_{N-1}, S \}$, $k_i$ integers, $k_i -
k_{i-1} \ge 1$, $k_0 \ge 0, k_{N-1} < S \}$, the   
problem can formally be stated as:
\begin{eqnarray}
& minimize_{K \in \textbf{K}} \; {W_q} \label{eqn:problem1} \\
& s.t. \; B \ge B_l \nonumber \\
& \sum_{j={k_0}}^{S} P(j) = 1 \nonumber \\
& equations \; (\ref{eqn:balance1}), (\ref{eqn:Fexpression1}), (\ref{eqn:Bexpression}),(\ref{eqn:Lexpression1}), (\ref{eqn:Wqexpression}). \nonumber  
\end{eqnarray}
Berman et al. \citeyear{Berman05a} refer to this problem as problem $P_1$. It is important to note that $B$, $F$ and $L$ are expected values and can be real-valued. Consequently, the constraint $B \ge B_l$ states that 
the expected number of workers in the back room resulting from the realization of any policy should be greater than or equal to the minimum expected number of back room
workers needed to complete all work in the back room. At any particular time
point, there may, in fact, be fewer than $B_l$ workers in the back room. 

As far as we are aware, the computational complexity of problem $P_1$
has not been determined. Berman et al. (p. 354) state that solving problem
$P_1$ ``exactly is extremely difficult since the constraints set (the
detailed balance equations) changes when the policy changes.''

\subsubsection{Berman et al.'s Heuristic\label{sec:heuristic}}
Berman et al. \citeyear{Berman05a} propose a heuristic method for the solution of this problem. 
This method is based on two corollaries and a theorem, which are stated and proved by the authors. 
These results are key to the understanding of the problem, and are, therefore, repeated below.

\begin{theorem} [Berman et al.'s Theorem 1] \label{thm:thm1}	
Consider two policies $K$ and $K'$ which are equal in all but one $k_i$. In particular, suppose that the value of $k'_J$ equals $k_J - 1$, 
for some $J$ from the set $\{0, ..., N-1\}$ such that $k_J-k_{J-1} \ge 2$, while $k'_i = k_i$ for all $i \ne J$. Then (a) $W_q(K) \ge W_q(K')$, (b) $F(K) \le F(K')$, (c) $B(K) \ge B(K')$.
\end{theorem}	
	
As a result of Theorem \ref{thm:thm1}, it can be seen that two policies exist which have special properties. Firstly, consider the policy\label{Khat} 
\[\hat K  = \left\{ k_0 = 0, k_1= 1, k_2 = 2,..., k_{N-1} = N-1, k_N = S \right\}. \] This policy results in the largest possible $F$, and 
the smallest possible $B$ and $W_q$. Because this policy yields the smallest possible expected waiting time, it is optimal if it is feasible. 
On the other hand, the smallest possible $F$ and the largest possible $W_q$ and $B$ are obtained by applying the policy\label{Kdoublehat} 
\[\hat {\hat K} = \left\{ k_0 = S - N, k_1 = S - N + 1,..., k_{N-1} = S - 1, k_N = S \right\}.\] 
Therefore, if this policy is infeasible, the problem (\ref{eqn:problem1}) is infeasible also.

Berman et al. propose the notions of eligible type 1 and type 2 components. An \emph{eligible type 1 component} is a switching point $k_i$ satisfying the condition that $k_i - k_{i-1} > 1$ for $0 < i < N$ or $k_i > 0$ for $i = 0$.
A switching point $k_i$ is an \emph{eligible type 2 component} 
if $k_{i+1} - k_i > 1$ for $0 \le i < N$. More simply, an eligible type 1 component is a $k_i$ variable which, if decreased by 1, will still be greater than $k_{i-1}$, while an eligible type 2 component is a $k_i$ variable which, if increased by $1$, will remain smaller than $k_{i+1}$. Eligible type 1 components and eligible type 2 components will further be referred to simply as type 1 and type 2 components, respectively.
	
Based on the definitions of policies $\hat K$ and $\hat {\hat K}$, the notions of type 1 and 2 components, 
and Theorem \ref{thm:thm1}, Berman et al. propose a heuristic, which has the same name as the problem it is used for, $P_1$:
\begin{enumerate}{\setlength{\itemsep}{0pt}}
\item Start with $K = \hat {\hat {K}}$.
\item If $B(K) < B_l$, the problem is infeasible. Otherwise, let $imb\_{W_q} = W_q(K)$ and $imb\_K = K$. Set $J = N$.\item Find the smallest $j^{*}$ s.t. $0 \le j^{*} < J$ and $k_{j^{*}}$ is a type 1 component. If no such $j^{*}$ exists, go to 5. Otherwise, set $k_{j^{*}}  = k_{j^{*}}  - 1$. If $B(K) < B_l$,  set $J = j^{*}$ and go to 5. If $B(K) \ge B_l$, go to 4.
\item If $W_q(K) < imb\_{W_q}$, let $imb\_ {W_q} = W_q(K)$ and $imb\_K = K$. Go to 3.
\item Find the smallest $j^{*}$ s.t. $0 \le j^{*} < J$ and $k_{j^{*}}$ is a type 2 component. If no such $j^{*}$ exists, go to 6. Otherwise, set $k_{j^{*}} = k_{j^{*}} + 1$. If $B(K) < B_l$, repeat 5. If $B(K) \ge B_l$, go to 4.
\item Stop and return $imb\_K$ as the best solution.
\end{enumerate}

Limiting the choice of $j^{*}$ to being between 0 and $J$, and resetting $J$ every time an infeasible policy is found, prevents the heuristic from entering an infinite cycle. The heuristic guarantees optimality only when the policy it returns is $\hat {K}$ or $\hat {\hat{K}}$.


Empirical results regarding the performance of heuristic $P_1$ are not 
presented in the paper by Berman et al. \citeyear{Berman05a}.  
In particular, it is not clear how close policies provided by $P_1$ are to the optimal policies.

\subsection{Summary of Parameters and Quantities of Interest}

Berman et al.'s model of problem $P_1$ requires five input parameters (Table \ref{tab:problemParameters}) and has expressions for calculating four main quantities of interest (Table \ref{tab:quantitiesOfInterest}), most of which are non-linear.

\begin{table}
\begin{center}
\begin{tabular}{|c|c|} \hline
Parameter & Meaning \\ \hline \hline
$S$ & front room capacity \\ \hline
$N$ & number of workers in the facility \\ \hline
$\lambda$ & arrival rate \\ \hline
$\mu$ & service rate \\ \hline
$B_l$ & expected number of workers required in the back room \\ \hline
\end{tabular}
\caption{\label{tab:problemParameters}Summary of problem parameters.}
\end{center}
\end{table}

\begin{table}
\begin{center}
\begin{tabular}{|c|c|c|} \hline
Notation & Meaning & Definition \\ \hline \hline
$F$ & expected number of workers in the front room & Equation (\ref{eqn:Fexpression1}) \\ \hline
$B$ & expected number of workers in the back room & Equation (\ref{eqn:Bexpression}) \\ \hline
$L$ & expected number of customers in the front room & Equation (\ref{eqn:Lexpression1}) \\ \hline
$W_q$ & expected customer waiting time & Equation (\ref{eqn:Wqexpression}) \\ \hline
\end{tabular}
\caption{\label{tab:quantitiesOfInterest}Summary of quantities of interest.}
\end{center}
\end{table}
%

%
%
%
%
%
%
%
%
%
\section{Constraint Programming Models}
\label{sec:cp_models}
Some work has been done on extending CP to stochastic problems 
\cite{Tarim06a,Tarim05b,Walsh02}. 
Our problem is different from the problems addressed in these papers because
all of the stochastic information can be explicitly encoded as constraints and expected values, and
there is no need for either stochastic variables or scenarios. 

It is known that creating an effective constraint programming model usually requires one to experiment with various problem representations \cite{Smith06}. Consequently, here we present, and experiment with, three alternative models for problem $P_1$ (Equation (\ref{eqn:problem1})). Although the first model is based directly on the formulation of Berman et al., further models were motivated by standard CP modelling techniques, such as the use of dual variables \cite{Smith06}.
We investigate the following three CP models:
\begin{itemize}
\item The \textit{If-Then} model is a CP version of the formal definition of
      Berman et al.  
\item The \textit{PSums} model uses a slightly different set of
      variables, and most of the constraints are based on closed-form 
      expressions derived from the constraints that are used in the
      \textit{If-Then} model. 
\item The \textit{Dual} model includes a set of dual decision variables in
      addition to the variables used in the \textit{If-Then} and \textit{PSums} models. 
	Most of the constraints of this model are expressed in terms of these dual variables. 
\end {itemize}

\subsection{Common Model Components}

There are a number of modelling components that are common to each of the
constraint models. Before presenting the models, we therefore present their
common aspects.

\subsubsection{Decision Variables}


All three of the proposed models have
a set of decision variables $k_i$, $i = 0, 1, \dots, N$, representing
the switching policy. Each $k_i$ from this set has the domain $[{i, i+1, \dots, S-N+i}]$ 
and has to satisfy the constraint $k_i < k_{i+1}$ (since the number of workers in the front room, $i$, 
increases only when the number of customers, $k_i$, increases).
Due to Berman et al.'s policy definition, $k_N$ must equal $S$.

\subsubsection{Additional Expressions}
All models include variables and constraints for the representation of the
balance equations (Equation (\ref{eqn:balance1})), and expressions for $F$,
the expected number of workers in the front room, and $L$, the expected
number of customers in the front room. However, these representations differ
slightly depending on the model, as noted below in Sections
\ref{subsec:ifThenModel}, \ref{subsec:pSumsModel} and
\ref{subsec:dualModel}. 
%

A set of auxiliary variables, $\beta Sum(k_i)$, defined as $\sum_{j=k_{i-1}+1}^{k_i} \beta_j$, for all $i$ from 1 to $N-1$, is included in each of the models (see Equations (\ref{eqn:calculatingPj})--(\ref{eqn:Xi1}) for the definition of $\beta_j$). These are necessary for representing Equation (\ref{eqn:restatementOfBetaConstraint}), which relates these variables to $P(k_0)$, a floating point variable with domain $[0..1]$ representing the probability of having $k_0$ customers in the facility. These auxiliary variables and constraint ensure that an assignment of all decision variables leads to a unique solution of the balance equations. We discuss the formal definition of these auxiliary variables in Section \ref{subsec:betaVariables}. 
\begin{eqnarray}
P(k_0)\sum_{i=0}^{N} \beta Sum(k_i) &=& 1 \label{eqn:restatementOfBetaConstraint}
\end{eqnarray}

The back room constraint, $B \ge B_l$, is stated in all models as $N-F \ge B_l$. The equation for $W_q$ is stated in all models as Equation (\ref{eqn:Wqexpression}).
\subsection{\textit{If-Then} Model\label{subsec:ifThenModel}}
The initial model 
includes the variables $P(j)$ for $j = k_0, k_0+1, \dots, k_1, k_1+1, \dots, k_N-1, k_N$, each
representing the steady-state probability of there being $j$ customers in the front
room. These floating point variables with domain $[0..1]$ have to satisfy
a system of balance equations (Equation (\ref{eqn:balance1}))
and are used to express $L$ and $F$.

The complete \textit{If-Then} model is presented in Figure \ref{fig:ifThenModel}.

\begin{figure}
\begin{eqnarray}
minimize & W_q & {} \nonumber \\
subject \;\; to & {} & \nonumber \\
k_i & < & k_{i+1} \indent \indent \indent \indent \indent \;\; \forall i \in \{0,1, \dots, N-1\}; \nonumber \\
k_N &=& S; \nonumber \\
(k_i \le j \le k_{i+1} - 1) & \rightarrow & P(j)\lambda = P(j+1)(i+1)\mu, \nonumber \\
{} & {} & \forall i \in \{0,1, \dots, N-1\}, \forall j \in \{0, 1, \dots, S-1\}; \nonumber \\
{} \nonumber \\
(j < k_0) &\le& (P(j) = 0), \indent \indent \indent \; \forall j \in \{0,1, \dots, S-N-1\}; \nonumber \\ 
\sum_{j=0}^{S} P(j) & = & 1; \nonumber \\ 
(k_0 = j) & \rightarrow & P(j)\sum_{i=0}^{N} \beta Sum(k_i) = 1, \nonumber \\ 
{} & {} & \indent \indent \indent \indent \indent \indent \;\;\;\; \forall j \in \{0, 1, \dots, S\}; \nonumber \\
L &=& \sum_{j=0}^{S} jP(j); \nonumber \\
(k_{i-1}+1 \le j \le k_{i}) & \rightarrow & r(i, j) = iP(j), \nonumber \\
{} & {} & \forall i \in \{1, 2, \dots, N\}, \forall j \in \{0, 1, \dots, S\}; \nonumber \\
(k_{i-1}+1 > j \;\; \vee \;\; j > k_{i}) & \rightarrow & r(i, j) = 0, \nonumber \\
{} & {} & \forall i \in \{1, 2, \dots, N\}, \forall j \in \{0, 1, \dots, S\}; \nonumber \\
F &=& \sum_{i=1}^{N} {\sum_{j=0}^{S} r(i, j)}; \nonumber \\
W_q &=& \frac{L}{\lambda (1 - P(k_N))} - \frac{1}{\mu}; \nonumber \\
N - F & \ge & B_l; \nonumber \\ 
auxiliary & {} & constraints. \nonumber 
\end{eqnarray} \nonumber
\caption{\label{fig:ifThenModel}Complete \textit{If-Then} Model}
\end{figure}

\subsubsection{Balance Equation Constraints}\label{sec:balEquationConstraints}

The balance equations are represented by a set of if-then constraints. For
example, the first balance equation, $P(j)\lambda = P(j+1)\mu$ for $j = k_0,
k_0+1, ..., k_1 - 1$, is represented by the constraint $(k_0 \le j \le k_1 -
1) \rightarrow P(j)\lambda = P(j+1)\mu$. Thus, somewhat inelegantly, an
if-then constraint of this kind has to be added for each $j$ between $0$ and
$S-1$ (inclusive) in order to represent one balance equation. In order to
represent the rest of these equations, this technique has to be applied for
each pair of switching points $k_i$, $k_{i+1}$ for $i$ from $0$ to
$N-1$. This results in a total of $N \times S$ if-then constraints.

The probabilities $P(j)$ also have to satisfy the constraint $\sum_{j = k_0}^{S} P(j) = 1$. The 
difficulty with this constraint is the fact that the sum starts at $j = k_0$, where $k_0$ is a decision 
variable. In order to deal with this complication, we add the meta-constraint $((j < k_0) \le (P(j) = 0))$ for each $j$ from the set $\{0, 1, \dots, S-N-1\}$.\footnote{We do not need to add 
this constraint for all $j$ from 0 to $S$ because the upper bound of the domain of $k_0$ is $S-N$.} This implies that all values of $P(j)$ 
with $j$ less than $k_0$ will be 0 and allows us to express the sum-of-probabilities constraint as
$\sum_{j=0}^{S} P(j) = 1$. 
 
%
\subsubsection{Expected Number of Workers Constraints}
A set of if-then constraints also has to be included
in order to represent 
Equation (\ref{eqn:Fexpression1}) 
as a constraint in our model. This is due to the dependence 
of this constraint on sums of variables between two switching points, which
are decision variables. 
%
%
More specifically, we add a set of variables $r(i, j)$ for representing the 
product of $i$ and $P(j)$ when $j$ is between $k_{i-1}+1$ and $k_{i}$, and  
the constraints $(k_{i-1}+1 \le j \le k_{i}) \rightarrow r(i, j) = iP(j)$ and
$(k_{i-1}+1 > j \;\; \vee \;\; j > k_{i}) \rightarrow r(i, j) = 0$ for all $i$ from 1 to $N$ and for all $j$
from 0 to $S$.
The total number of these if-then constraints is $2N(S+1)$.
$F$ can then be simply stated as a sum over the indices $i$ and $j$ of variables $r(i,j)$. 
\subsubsection{Expected Number of Customers Constraint}
$L$ is defined according to Equation (\ref{eqn:Lexpression1}). Since the meta-constraint
$((j < k_0) \le (P(j) = 0))$ has been added to the model in order to ensure that $P(j) = 0$ 
for all $j < k_0$, the constraint for $L$ can be simply stated as
the sum of the products of $j$ and $P(j)$ over all $j$ from 0 to $S$:
\begin{eqnarray}
L &=& \sum_{j=0}^{S} jP(j). \label{eqn:simplerL}
\end{eqnarray}

\subsubsection{Auxiliary Variables and Constraints\label{subsec:betaVariables}}

The model includes a set of $N-2$ auxiliary expressions $X_i$ for all $i$ from 3 to $N$ ($X_1$ and $X_2$ are always equal to 1). Instead of including $S$ $\beta_j$ variables (refer to Equation (\ref{eqn:calculatingPj}) for the definition of $\beta_j$), we use $N+1$ continuous auxiliary variables $\beta Sum(k_i)$ with domain $[0 \dots 1+S{\left(\frac{\lambda}{\mu}\right)}^S]$, which represent sums of $\beta_j$ variables between $j = k_{i-1}+1$ and $j = k_i$ (inclusive). $\beta Sum(k_0)$ is constrained to equal 1, while the rest of these variables are defined according to Equation (\ref{eqn:betaSumEquation}). 
The validity of this equation is proved in the appendix.
%
\begin{eqnarray}\label{betaSum1}
\beta Sum(k_i) = \sum_{j=k_{i-1}+1}^{k_i} \beta_j
=\left\{ 
\begin{array}{rr}
X_i{\left(\displaystyle\ratioSimple\right)}^{k_{i-1}-k_0+1}\left(\displaystyle\iRatio\right) \left[\frac{1-{\left(\displaystyle\lambdamuiRatio\right)}^{k_i-k_{i-1}}}{1-\left(\displaystyle\lambdamuiRatio\right)} \right] & \mbox{if } \lambdamuiRatio \neq 1\\
\\
X_i{\left(\displaystyle\ratioSimple\right)}^{k_{i-1}-k_0+1}\left(\displaystyle\iRatio\right) (k_i-k_{i-1}) & otherwise. 
\end{array} \right. \label{eqn:betaSumEquation}	\\
\nonumber \\
\forall i \in \{1, \dots, N\}; \nonumber
\end{eqnarray}
The sum $\sum_{j=k_0}^{k_N}\beta_j$ can then be expressed as $\sum_{i=0}^{N} \beta Sum(k_i)$. The requirement that $P(k_0)\times \sum_{j={k_0}}^{k_N}\beta_j$ has to equal 1 is stated as a set of if-then constraints $(k_0 = j)$ $\rightarrow$ $P(j) \times$ \\ $\sum_{i=0}^{N} \beta Sum(k_i)$ $=$ $1$ for all $j \in \{0,1, \dots, S\}$. 

In summary, the $auxiliary \; constraints$ that are present in the model are: $\beta Sum(k_0) = 1$, 
Equation (\ref{eqn:betaSumEquation}) and
$X_i = \prod_{g=1}^{i-1}{{\left( \frac{1}{g} \right)}^{k_g-k_{g-1}}} \; \forall i \in \{3, \dots, N\}$.  

%

%

The \textit{If-Then} model includes a total of $3NS + 2N + S + 1$ if-then constraints, which
are often ineffective in search because propagation only occurs either
when the left-hand side is satisfied or when the right-hand side becomes
false. Consequently, the next model attempts to avoid, as much as possible,
the use of such constraints.
\subsection{\textit{PSums} Model\label{subsec:pSumsModel}}
%
%
The second CP model is based on closed-form expressions derived from the balance equations (the details of the derivations are provided in 
the appendix). 
The set of $P(j)$ variables from the
formulation of Berman et al. is replaced by a set of $PSums(k_i)$\label{PSums1} variables
for $i = 0, \dots, N-1$, together with a set of probabilities $P(j)$\label{Pki1} for $j =
k_0, k_1, k_2, \dots, k_N$. Note that $P(j)$ is defined for each
switching point only, not for all values from $\{0, 1, \dots, S\}$. The $PSums(k_i)$
variable represents the sum of all probabilities between $k_i$ and
$k_{i+1}-1$. 
%

The complete \textit{PSums} model is presented in Figure \ref{fig:pSumsModel}. The remainder of this section provides details of this model.

\subsubsection{Probability Constraints}
Balance equations are not explicitly stated in this model. 
However, expressions for $P(k_i)$ and $PSums(k_i)$ are derived in such a way that the
balance equations are satisfied. 
%
%
The $PSums(k_i)$ variables are defined in Equation (\ref{eqn:pSums}). Equation (\ref{eqn:switchProb}) is a recursive 
formula for computing $P(k_{i+1})$.
 
%
%
\begin{eqnarray} \nonumber
PSums(k_i) &=& \sum_{j=k_i}^{k_{i+1}-1} P(j) \\ \nonumber
\\ 
{} &=& \left\{ 
\begin{array}{rr}
P(k_i) \frac{\displaystyle1-{\left[\displaystyle\ratio\right]}^{k_{i+1}-k_i}}{\displaystyle1-\displaystyle\ratio} & \mbox{if } \ratio \neq 1\\
\\
P(k_i) (\diff)  & otherwise. 
\end{array} \right. \label{eqn:pSums} 	
\end{eqnarray}
\begin{eqnarray}
P(k_{i+1}) &=& {\left(\ratio\right)}^{k_{i+1}-k_i}P(k_i), \;\;\; \forall i \in \{0, 1, \dots, N-1\}. \label{eqn:switchProb}
\end{eqnarray}
Additionally, the probability variables have to satisfy the constraint $\sum_{i=0}^{N-1} PSums(k_i) + P(k_N) = 1$. 
\begin{figure}\label{fig:PSumsModel}
\begin{eqnarray}
minimize & W_q & {} \nonumber \\
subject \;\; to & {} & \nonumber \\
k_i & < & k_{i+1} \indent \indent \indent \indent \indent \indent \indent \indent \indent \indent \indent \; \forall i \in \{0,1, \dots, N-1\}; \nonumber \\
k_N &=& S; \nonumber \\
P(k_{i+1}) &=& {\left(\ratio\right)}^{k_{i+1}-k_i}P(k_i), \nonumber \\
{} & {} & \indent \indent \indent \indent \indent \indent \indent \indent \indent \indent \indent \indent \;\; \forall i \in \{0, 1, \dots, N-1\}; \nonumber \\
{} \nonumber \\
PSums(k_i) &=& \left\{ 
\begin{array}{rr}
P(k_i) \frac{\displaystyle1-{\left[\displaystyle\ratio\right]}^{k_{i+1}-k_i}}{\displaystyle1-\displaystyle\ratio} & \mbox{if } \ratio \neq 1\\
\\
P(k_i) (\diff) & otherwise 
\end{array} \right. ,\nonumber \\
{} & {} & \indent \indent \indent \indent \indent \indent \indent \indent \indent \indent \indent \indent \;\; \forall i \in \{1, 2, \dots, N-1\}; \nonumber \\
\sum_{i=0}^{N-1} PSums(k_i) &{+}& P(k_N) =  1; \nonumber \\
P(k_0) &{\times}& \sum_{i=0}^{N} \beta Sum(k_i) = 1; \nonumber \\ 
F &=& \sum_{i=1}^{N}{i\left[PSums(k_{i-1})-P(k_{i-1})+P(k_i) \right];} \nonumber \\
L &=& \sum_{i = 0}^{N-1} L(k_i) + k_NP(k_N), \nonumber \\
L(k_i) &=& k_iPSums(k_i)+P(k_i)\ratio \nonumber \\
{} & {\times} & \frac{ {\left(\ratio\right)}^{k_{i+1}-k_i-1}(k_i-k_{i+1}) + {\left(\ratio\right)}^{k_{i+1}-k_i}(k_{i+1}-k_i-1)+1 }{{\left(1-\ratio\right)}^2}, \nonumber \\
{} & {} & \nonumber \\
{} & {} & \indent \indent \indent \indent \indent \indent \indent \indent \indent \indent \indent \indent \; \forall i \in \{0,1, \dots, N-1\}; \nonumber
\end{eqnarray}
\begin{eqnarray}
W_q &=& \frac{L}{\lambda (1 - P(k_N))} - \frac{1}{\mu}; \nonumber \\
N - F & \ge & B_l; \nonumber \\ 
auxiliary & {} & constraints. \nonumber
\end{eqnarray}
\caption{\label{fig:pSumsModel}Complete \textit{PSums} Model}
\end{figure}
\subsubsection{Expected Number of Workers Constraint}
$F$, the expected number of workers in the front room, can be expressed in terms of $P(k_i)$ and $PSums(k_i)$ as shown in Equation (\ref{eqn:PSumsFexpression}):
\begin{eqnarray}
F &=& \sum_{i=1}^{N}{i\left[PSums(k_{i-1})-P(k_{i-1})+P(k_i) \right]}. \label{eqn:PSumsFexpression}
\end{eqnarray}
\subsubsection{Expected Number of Customers Constraints}
The equation for $L$ is
\begin{eqnarray}
L &=& \sum_{i = 0}^{N-1} L(k_i) + k_NP(k_N) \label{eqn:PSumsLexpression}
\end{eqnarray}
where\label{Lki1}
\begin{eqnarray} \nonumber
L(k_i) &=& k_iPSums(k_i)+P(k_i)\ratio \\ \nonumber 
{} & {} & \times \left[\frac{ {\left(\ratio\right)}^{k_{i+1}-k_i-1}(k_i-k_{i+1}) + {\left(\ratio\right)}^{k_{i+1}-k_i}(k_{i+1}-k_i-1)+1 }{{\left(1-\ratio\right)}^2} \right]. \nonumber
\end{eqnarray}

\subsubsection{Auxiliary Constraints}
The $auxiliary \; constraints$ are exactly the same as in the \textit{If-Then} model. However, 
the requirement that $P(k_0)\times \sum_{j={k_0}}^{k_N}\beta_j$ has to equal 1 is stated as $P(k_0) \sum_{i=0}^{N} \beta Sum(k_i) = 1$, rather than as a set of if-then constraints, because this model has an explicit closed-form expression for the variable $P(k_0)$. 
  

\subsection{\textit{Dual} Model\label{subsec:dualModel}}
The problem can be alternatively formulated using variables $w_j$\label{wj}, which
represent the number of workers in the front room when there are $j$
customers present. The $w_j$ variables can be referred to as the \emph{dual} variables because, 
compared to the $k_i$'s, 
the roles of variables and values are reversed \cite{Hnich04,Smith06}. As stated by Smith, the use of dual variables in a constraint programming model is beneficial if some constraints of the problem are easier to express using these new variables. This is the case for our problem -- in fact, the use of dual variables allows us to significantly reduce the number of if-then constraints necessary for stating relations between the probabilities.
 
In our \textit{Dual} model, 
there are $S+1$ $w_j$ variables, each
with domain $[0, 1, \dots, N]$. These variables have to satisfy the following equations:
$w_0 = 0$, $w_S = N$ and $w_j \le w_{j+1}$ for all $j$ from $0$ to $S-1$.
Additionally, the complete set of $k_i$ variables is included in this model, since some 
constraints are easier to express using the $k_i$'s rather than the $w_j$'s.

The complete \textit{Dual} model is presented in Figure \ref{fig:dualModel}, while the details are discussed below.

\begin{figure}
\begin{eqnarray}
minimize & W_q & {} \nonumber \\
subject \;\; to & {} & \nonumber \\
w_0 = 0,& {} & w_S = N \nonumber \\
w_j &\le& w_{j+1} \indent \indent \indent \indent \indent \indent \indent \indent \indent \indent \indent \; \forall j \in \{0,1, \dots, S-1\}; \nonumber \\
k_i & < & k_{i+1} \indent \indent \indent \indent \indent \indent \indent \indent \indent \indent \indent  \;\; \forall i \in \{0,1, \dots, N-1\}; \nonumber \\
k_N &=& S; \nonumber \\
w_j < w_{j+1} \;\; & \leftrightarrow & \;\; k_{w_{j}} = j  \indent \indent \indent \indent \indent \indent \indent \indent \indent \indent \forall j \in \{0, 1, \dots, S-1\}; \nonumber \\
w_j = w_{j+1} \;\; & \leftrightarrow & \;\; k_{w_{j}} \neq j \;\;\; \indent \indent \indent \indent \indent \indent \indent \indent \indent \;\; \forall j \in \{0, 1, \dots, S-1\}; \nonumber \\
w_j = i \;\; & \leftrightarrow & \;\; k_{i-1}+1 \le j \le k_i \indent \;\;\;\; \forall i \in \{1, 2, \dots, N\}, \forall j \in \{0, 1, \dots, S\}; \nonumber \\ 
P(j)\lambda &=& P(j+1)w_{j+1}\mu \indent \indent \indent \indent \indent \indent \indent \indent \;\; \forall j \in \{0, 1, \dots, S-1\}; \nonumber \\
(k_0 = j) & \rightarrow & P(j)\sum_{i=0}^{N} \beta Sum(k_i) = 1, \;\; \indent \indent \indent \indent \indent \;\;  \forall j \in \{0, 1, \dots, S\}; \nonumber \\
(j < k_0) &\le& (P(j) = 0), \indent \indent \indent \indent \indent \indent \indent \;\;\;\; \forall j \in \{0,1, \dots, S-N-1\}; \nonumber \\
\sum_{j=0}^{S} P(j) & = & 1; \nonumber \\ 
F &=& \sum_{j = 0}^{S} w_jP(j); \nonumber \\
 \nonumber \\
L &=& \sum_{j=0}^{S} jP(j); \nonumber \\
W_q &=& \frac{L}{\lambda (1 - P(k_N))} - \frac{1}{\mu}; \nonumber \\
N - F & \ge & B_l; \nonumber \\ 
{} \nonumber \\
auxiliary & {} & constraints.  \nonumber
\end{eqnarray}
\caption{\label{fig:dualModel}Complete \textit{Dual} Model}
\end{figure}

\subsubsection{Probability Constraints}
Given the dual variables, the balance equations can be restated as
\begin{eqnarray}
P(j)\lambda &=& P(j+1)w_{j+1}\mu, \;\;\; \forall j \in \{0, 1, \dots, S-1\}.	\label{eqn:dualBalanceEquations}
\end{eqnarray}		
This formulation of the balance equations avoids the inefficient if-then constraints. The rest of the 
restrictions on the 
probability variables are stated in terms of the $k_i$ variables, as in the \textit{If-Then} model. 
%
In particular, the constraints $\sum_{j=0}^{S} P(j) = 1$ and $((k_0 > j) \le (P(j) = 0))$ $\forall j \in \{0, \dots, S-N-1\}$ 
are present in this model.
\subsubsection{Channelling Constraints}
In order to use redundant variables, a set of channelling constraints has to be added to the model to ensure that an assignment of values to one set of variables will lead to a unique assignment of variables in the other set. The following channelling constraints are included:
\begin{eqnarray}
w_j < w_{j+1} &\leftrightarrow& k_{w_{j}} = j \qquad \qquad \qquad \forall j \in \{0, 1, \dots, S-1\}, \label{eqn:dualWjExpression1} \\
%
w_j = w_{j+1} &\leftrightarrow& k_{w_{j}} \neq j \qquad \qquad \qquad \forall j \in \{0, 1, \dots, S-1\}, \label{eqn:dualWjExpression2} \\ 
%
w_j = i &\leftrightarrow& k_{i-1}+1 \le j \le k_i \qquad \forall j \in \{0, 1, \dots, S\}, \;\; \forall i \in \{1, \dots, N\}. \label{eqn:dualWjExpression3} 
\end{eqnarray}
Constraints (\ref{eqn:dualWjExpression1}) and (\ref{eqn:dualWjExpression2}) are redundant given the constraint $w_j \le w_{j+1}$. However, such redundancy can often lead to increased propagation \cite{Hnich04}. One direction for future work is examining the effect that removing one of these constraints may have on the performance of the program.
\subsubsection{Expected Number of Workers Constraint}
The expression for the expected number of workers in the front room is $F = \sum_{j = 0}^{S} w_jP(j)$. 
%
%
\subsubsection{Expected Number of Customers Constraint}
The constraint used to express $L$ is identical to the one used in the \textit{If-Then} model: $L = \sum_{j=0}^{S} jP(j)$. 
%
This equation is valid because all $P(j)$ for $j < k_0$ are constrained to be 0.
%
%
\subsubsection{Auxiliary Constraints}
The $auxiliary \; constraints$ present in this model are exactly the same as
in the \textit{If-Then} model. The requirement that $P(k_0)\times
\sum_{j={k_0}}^{k_N}\beta_j = 1$ is also stated as a set of if-then constraints $(k_0 = j) \rightarrow P(j)\sum_{i=0}^{N} \beta Sum(k_i) = 1$ for all $j \in \{0,1, \dots, S\}$.

%
%
%

%
%
%
%
%
\begin{table}
\begin{center}
\begin{tabular}{|l|c|c|c|c|} \hline
Statistic/Model & \textit{If-Then} & \textit{PSums} & \textit{Dual} \\ \hline \hline
\# of decision variables & $N+1$ & $N+1$ & $N+S+2$ \\ \hline
\# of probability variables & $S+1$ & $2N+1$ & $S+1$ \\ \hline
\# of probability constraints & $N(S-1)+2S+2$ & $2N+1$ & $3S-N+2$ \\ \hline
\# of constraints for $F$ & $2N(S+1)+1$ & 1 & 1 \\ \hline
\# of constraints for $L$ & 1 & $N+1$ & 1 \\ \hline
\# of if-then constraints & $3NS+2N+S + 1$ & 0 & $S+1$ \\ \hline
total \# of constraints & $3NS + 4N + 2S + 6$ & $6N + 5$ & $NS + 3N + 6S + 8$ \\ \hline
\end{tabular}
\caption{\label{tab:modelSummary}Summary of the main characteristics of the three proposed CP models.}
\end{center}
\end{table}

\subsection{Summary\label{sec:modelSummary1}}
Table \ref{tab:modelSummary} presents a summary of the number of variables
and constraints in each of the 
three proposed models. 
It can be seen that the \textit{PSums} model has a smaller number of probability variables and
constraints but a slightly larger number of constraints for representing $L$ than the other two models, and no if-then constraints. The \textit{Dual} model has a larger number of decision variables than the \textit{If-Then} and \textit{PSums} models. This does not imply that the search space is bigger in this model because the two sets of variables are linked by channelling constraints and so are assigned values via propagation. The \textit{Dual} allows the simplest representations of $F$ and $L$, each requiring only one constraint. The number of probability constraints in the \textit{Dual} is smaller than or equal to the number of such constraints in the \textit{If-Then} model and greater than in the \textit{PSums} model. However, the actual representation of these constraints is the most straightforward in the \textit{Dual} since it neither requires if-then constraints nor closed-form expressions. 

It is hard to determine, simply by looking at Table \ref{tab:modelSummary}, which of the models will be more efficient since, in CP, a larger number of constraints and/or variables
may actually lead to more propagation and to a more effective model \cite{Smith06}. However, it is known that if-then constraints
do not propagate well, and, since the difference in the number of these constraints between the \textit{PSums} model
and both the \textit{If-Then} model and the \textit{Dual} model is quite significant, one may expect the \textit{PSums}
model to have some advantage over the other two models.

In fact, preliminary experiments with all three models showed poor performance (see Table \ref{tab:summ_res1} in Section \ref{sec:expResults}). 
Due to the complexity of constraints relating
the decision variables to variables representing probabilities, there was little constraint propagation, and, essentially, search was required to explore the entire branch-and-bound tree.
As a consequence, in the following two sections we examine dominance rules \cite{Beck04f,Smith05a} 
and shaving \cite{Caseau96,Martin96}, two stronger inference forms used in CP. 
In Section \ref{sec:discussion}, we investigate why the models without dominance rules and shaving need to search the whole tree in order to prove optimality, and also discuss differences in the performance of the models based on our experimental results.

%
%
%
%
%
\section{Dominance Rules\label{sec:dominanceRules}}
A dominance rule is a constraint that forbids assignments of values to variables which are known to be sub-optimal \cite{Beck04f,Smith05a}. 
For problem $P_1$, a dominance rule states that, given a feasible solution, $K$, all further solutions have to have at least one switching point assigned a lower value than the value assigned to it in $K$. In other words, given two solutions $K$ and $K'$, if the $W_q$ value resulting from policy $K'$ is smaller than or equal to the $W_q$ value resulting from $K$, then there have to exist switching points $k'_i$ and $k_i$ in $K'$ and $K$, respectively, satisfying the condition $k'_i < k_i$. The following theorem states the dominance rule more formally.

\begin{theorem}[Dominance Rule] \label{thm:dominanceRuleTheorem}
Let $K = (k_0, k_1, \dots, k_N)$ and $K' = (k'_0, k'_1, \dots, k'_N)$ be two policies such that $k_0 = k'_0 = 0$, $k_1 = k'_1 = 1, \dots, k_{J-1} = k'_{J-1} = {J-1}$ and $k_J \neq k'_J$ (i.e. at least one of $k_J$, $k'_J$ is strictly greater than $J$) for some $J \in \{0, 1, \dots, N-1\}$.
Let $W_q(K)$ and $W_q(K')$ denote the expected waiting times resulting from the two policies $K$ and $K'$, respectively. If $W_q(K') < W_q(K)$, then there exists $i \in \{J, J+1, J+2, \dots , N-1\}$ for which $k'_i < k_i$. 
\end{theorem}
\begin{proof} [By Contraposition]
We prove the contrapositive of the statement in Theorem \ref{thm:dominanceRuleTheorem}: we assume that there does not exist $i \in \{J, J+1, \dots, N-1\}$ such that $k'_i < k_i$ and show that, given this assumption, $W_q(K')$ has to be greater than or equal to $W_q(K)$.

Assume no $i \in \{J, J+1, \dots, N-1\}$ exists for which $k'_i < k_i$. Then one of the following is true: \\
(a)	$k_n = k'_n$ for all $n \in \{J, J+1, \dots, N-1\}$, or \\
(b)	there exists at least one $j \in \{J, J+1, \dots, N-1\}$ such that $k'_j > k_j$, and the values of the rest of the switching points are the same in the two policies.

Case (a) implies that $K'$ and $K$ are the same policy, and so $W_q(K) = W_q(K')$.

To prove (b), suppose there exists exactly one $j \in \{J, J+1, J+2, \dots, N-1\}$ such that $k'_j > k_j$, and $k_n = k'_n$ for all $n \in \{J, \dots, N-1 \}\setminus\{j\}$. Then $K$ and $K'$ are different in the value of exactly one switching point. Consequently, by Theorem \ref{thm:thm1}, $W_q(K') \ge W_q(K)$. Similarly, by applying Theorem \ref{thm:thm1} several times, the result generalizes to cases when there exists more than one $j$ such that $k'_j > k_j$. 

Therefore, if no $i \in \{J, J+1, \dots, N-1\}$ exists for which $k'_i < k_i$, it follows that $W_q(K') \ge W_q(K)$.
In other words, if $W_q(K') < W_q(K)$, then there exists $i \in \{J, J+1, \dots, N-1\}$ for which $k'_i < k_i$.
\end{proof}		

Theorem \ref{thm:dominanceRuleTheorem} implies that, given a feasible policy 
$(0, 1, \dots, k^*_i, k^*_{i+1}, \dots, k^*_{N-1}, S)$ where all switching points with index 
$i$ or greater are assigned values that are strictly greater than their lower bounds, we know that a solution with smaller $W_q$ has to 
satisfy the constraint $((k_i < k^*_i)\; \vee\; (k_{i+1} < k^*_{i+1})\; \vee\; \dots\; \vee\; (k_{N-1} < k^*_{N-1}))$. 
Therefore, in order to implement the dominance rule, we add such a constraint during search every time a feasible policy is found, which should lead to a reduction in the size of the search space. Section \ref{sec:expResults} presents experimental results regarding the usefulness of this technique.

\section{\label{sec:shaving}Shaving} 
%
%

Shaving is an inference method that temporarily adds constraints to the
model, performs propagation, and then soundly prunes variable domains based
on the resulting state of the problem \cite{Demassey05,vanDongen06}.  For
example, a simple shaving procedure may be based on the assignment of a
value $a$ to some variable $x$. If propagation following the assignment
results in a domain wipe-out for some variable, the assignment is
inconsistent, and the value $a$ can be removed from the domain of $x$
\cite{Demassey05,vanDongen06}. In a more general case, both the temporary
constraint and the inferences made based on it can be more complex. Shaving
has been particularly useful in the job-shop scheduling domain, where it is
used to reduce the domains of start and end times of operations
\cite{Caseau96,Martin96}. For such problems, shaving is used either as a
domain reduction technique before search, or is incorporated into
branch-and-bound search so that variable domains are shaved after each
decision \cite{Caseau96}. 
 
In a shaving procedure for our problem, we temporarily assign a particular value to a switching point variable, while the rest of the variables are assigned either their maximum or their minimum possible values. Depending on whether the resulting policies are feasible or infeasible, new bounds for the switching point variables may be derived. 

The instance $N = 3$, $S = 6$, $\lambda = 15$, $\mu = 3$, $B_l = 0.32$
is used below for illustration purposes. 
Policy $\hat {K}$, which always yields the smallest possible $W_q$, for this
instance is $(k_0, k_1, k_2, k_3) = (0, 1, 2, 6)$ and policy $\hat {\hat
{K}}$, which always yields the greatest possible $W_q$, is $(3, 4, 5,
6)$. Thus, the initial domains of the switching points are $[0..3]$,
$[1..4]$, $[2..5]$ and $[6]$ for $k_0$, $k_1$, $k_2$ and $k_3$,
respectively. At any step, shaving may be able to reduce the domains of one
or more of these variables.
\subsection{$B_l$-based Shaving Procedure}
%
%
The initial shaving procedure consists of two cases in which either the
upper or the lower bounds of variables may be modified. In the first case,
the constraint $k_i = \min(k_i)$, where $\min(k_i)$ is the smallest value
from the domain of $k_i$, is temporarily added to the problem for some
particular value of $i$ between $0$ and $N$. All other switching points are
assigned their maximum possible values using the function
\textit{gMax}. Given an array of variables, the function \textit{gMax}
assigns the maximum possible values to all of the variables that do not yet
have a value, returning true if the resulting assignment is feasible, and
false otherwise. The maximum possible values are not necessarily the upper
bound values in the domains of the corresponding variables, rather they are
the highest values in these domains that respect the condition that $k_n <
k_{n+1}$, $\forall n \in \{0, ..., N-1\}$. In the example, if $k_1$ is
assigned the value $1$, while the rest of the variables are unbound,
\textit{gMax} would result in policy $(0, 1, 5, 6)$, which has a feasible
$B$ value of $0.508992$, and thus true would be returned.

Recall that an assignment is infeasible when it yields a $B$ value which is smaller than $B_l$. When the policy resulting from the addition of $k_i = \min(k_i)$ and the use of \textit{gMax} is infeasible, and $\min(k_i)+1 \le \max(k_i)$, the constraint $k_i > \min(k_i)$ can be added to the problem: if all variables except $k_i$ are set to their maximum values, and the problem is infeasible, then in any feasible policy $k_i$ must be greater than $\min(k_i)$. Such reasoning is valid since Theorem \ref{thm:thm1} states that increasing the value of a switching point will increase $B$. Note that if the solution is feasible, it should be recorded as the best-so-far solution if its $W_q$ value is smaller than the $W_q$ value of the previous best policy. For easier reference, this part of the shaving procedure will be referred to as the \textit{gMax} case.

In the second case, the constraint $k_i = \max(k_i)$ is added to the problem for some $i$ between $0$ and $N$, where $\max(k_i)$ is the maximum value from the domain of $k_i$. The rest of the variables are assigned the minimum values from their domains using the function \textit{gMin}. These assignments are made in a way that respects the constraints $k_n < k_{n+1}$,
$\forall n \in \{0, ..., N-1\}$. If the resulting policy is feasible, then the constraint $k_i < \max(k_i)$ can be permanently added to the problem, assuming $\max(k_i)-1 \ge \min(k_i)$. Since all variables except $k_i$ are at their minimum values already, and $k_i$ is at its maximum, it must be true, again by Theorem \ref{thm:thm1}, that in any better solution the value of $k_i$ has to be smaller than $\max(k_i)$. This case will be further referred to as the \textit{gMin} case.

In both cases, if the inferred constraint violates the current upper or
lower bound of a $k_i$, then the best policy found up to that point is
optimal. Whenever the domain of a switching point is modified as a result of
inferences made during the \textit{gMax} or \textit{gMin} case, all of the
switching points need to be re-considered. If the
domain of one variable is reduced during a particular shaving iteration,
some of the temporary constraints added in the next round of shaving will be
different from the ones used previously, and, consequently, new inferences
may be possible. Thus, the shaving procedure terminates when optimality is
proved or when no more inferences can be made.

Consider the example mentioned above. Suppose the constraint $k_0 = 0$ is added to the problem, and all the rest of the variables are assigned their maximum possible values (using \textit{gMax}). The resulting policy is $(k_0, k_1, k_2, k_3) = (0, 4, 5, 6)$. This policy yields a $B$ value of 0.63171, which implies that this policy is feasible because $0.63171 > 0.32 = B_l$, and so no domain reductions can be inferred. The constraint $k_0 = 0$ is then removed. Since the domain of $k_0$ has not been modified, the procedure considers the next variable. Thus, the constraint $k_1 = 1$ is added, and all the other variables are again set to their maximum values. The resulting policy is $(0, 1, 5, 6)$, which is also feasible since its $B$ value is $0.508992$. The constraint $k_1 = 1$ is then removed, and $k_2 = 2$ is added. When all the other variables are set to their maximum values, the resulting policy is $(0, 1, 2, 6)$. This policy yields a $B$ value of 0.1116577, which is smaller than $B_l$. Thus, this policy is infeasible, and the constraint $k_2 > 2$ is added to the problem. This changes the domain of $k_2$ to $[3..5]$. Whenever the domain of a variable is reduced, the next shaving step considers the same switching point, and so the next constraint to be added is $k_2 = 3$.

Now, consider the \textit{gMin} case and assume that all variables have their full initial domains. Suppose the constraint $k_0 = 3$ is added to the problem. All the rest of the variables are assigned their smallest possible values consistent with $k_0 = 3$. Thus, the policy $(3, 4, 5, 6)$ is considered. This policy has a $B$ value of $0.648305$ and is feasible (it is in fact ${\hat {\hat K} }$, so if it were infeasible, the problem would have been infeasible). The value of $k_0$ in any better solution has to be smaller than $3$, and so the domains of the variables become $[0..2]$, $[1..4]$, $[2..5]$, and $[6]$. The constraint $k_0 = 3$ is removed, and, since the domain of $k_0$ has been modified, the constraint $k_0 = 2$ is added next. The policy that is considered now is $(2, 3, 4, 6)$. This policy is also feasible, and so the domains become $[0..1]$, $[1..4]$, $[2..5]$, $[6]$. The temporary constraint $k_0 = 2$ is removed, and the next one added is $k_0 = 1$. The corresponding policy assigned by \textit{gMin} is infeasible, and no domain reductions are made. Since the addition of $k_0 = 1$ did not result in any domain reductions, there is no need to reconsider the variable $k_0$ until after all other switching points have been looked at. Consequently, the next temporary constraint to be added is $k_1 = 4$.

In the complete $B_l$-based shaving procedure, we can start either with the \textit{gMin} or the \textit{gMax} case. Since policies considered in the \textit{gMin} case will generally have smaller waiting time than ones considered in the \textit{gMax} case, it may be beneficial to start with the \textit{gMin} case. This is the approach we take. 

Our complete $B_l$-based shaving algorithm is presented in Figure \ref{fig:blShaving1}. It is assumed in all of the algorithms presented that the functions ``add($constraint$)'' and ``remove($constraint$)'' add and remove $constraint$ to and from the model, respectively.

\begin{figure}
{\bf Algorithm 1: {\textit{$B_l$-based Shaving}}}\\
{\bf Input:} $S$, $N$, $\mu$, $\lambda$, $B_l$ (problem instance parameters); $bestSolution$ (best solution found so far) \\
{\bf Output:} $bestSolution$ with (possibly) modified domains of the variables $k_i$, or $bestSolution$ with proof of its optimality
\begin{tabbing}
\hspace{0.15in} \= \hspace{0.15in} \= \hspace{0.15in} \= \hspace{0.15in} \= \hspace{0.15in} \= \hspace{0.75in} \\
{\bf while} (there are domain changes) \\
\> {\bf for} all $i$ from $0$ to $N-1$ \\
\>\> {} \\
\>\> {\bf while} (shaving successful for $Domain(k_i)$) \\
\>\>\> add( $k_i = \max(Domain(k_i))$ ) \\
\>\>\> {\bf if} (\textit{gMin}) \\
\>\>\>\> {\bf if} (new best solution found) \\
\>\>\>\>\> $bestSolution = currentSolution$; \\ 
\>\>\>\> {\bf if} ( $\max(Domain(k_i))-1 \ge \min(Domain(k_i))$ ) \\
\>\>\>\>\> add( $k_i < \max(Domain(k_i))$ ) \\
\>\>\>\> {\bf else} \\
\>\>\>\>\> return $bestSolution$; stop, optimality has been proved \\
\>\>\> remove( $k_i = \max(Domain(k_i))$ ) \\
\>\> {} \\
\>\> {\bf while} (shaving successful for $Domain(k_i)$) \\
\>\>\> add( $k_i = \min(Domain(k_i))$  ) \\
\>\>\> {\bf if} (\textit{gMax}) \\
\>\>\>\> {\bf if} (new best solution found) \\
\>\>\>\>\> $bestSolution = currentSolution$; \\ 
\>\>\> {\bf else} \\
\>\>\>\> {\bf if} ( $\min(Domain(k_i))+1 \le \max(Domain(k_i))$ ) \\
\>\>\>\>\> add( $k_i > \min(Domain(k_i))$ ) \\
\>\>\>\> {\bf else} \\
\>\>\>\>\> return $bestSolution$; stop, optimality has been proved \\
\>\>\> remove( $k_i = \min(Domain(k_i))$ ) \\
\end{tabbing}
\caption{\label{fig:blShaving1}$B_l$-based shaving algorithm}
\end{figure}

Upon the completion of this shaving procedure, the constraint $W_q \le
bestW_q$, where $bestW_q$ is the value of the best solution found up to that
point, is added ($W_q \le bestW_q$ rather than $W_q < bestW_q$ is added
because of numerical issues with testing equality of floating point numbers). However, although such a constraint rules out policies with higher $W_q$ as infeasible, it results in almost no propagation of the domains of the decision variables and does little to reduce the size of the search tree. In order to remedy this problem, another shaving procedure, this time based on the constraint $W_q \le bestW_q$ is proposed in the next sub-section. The issue of the lack of propagation of the domains of $k_i$ from the addition of this constraint will be discussed in more detail in Section \ref{sec:lackofpropagation}.
\subsection{$W_q$-based Shaving Procedure}
%
%
%
%

\begin{figure}
{\bf Algorithm 2: {\textit{$W_q$-based Shaving}}}\\
{\bf Input:} $S$, $N$, $\mu$, $\lambda$, $B_l$ (problem instance parameters); $bestSolution$ (best solution found so far) \\
{\bf Output:} $bestSolution$ with (possibly) modified domains of the variables $k_i$, or $bestSolution$ with proof of its optimality
\begin{tabbing}
\hspace{0.15in} \= \hspace{0.15in} \= \hspace{0.15in} \= \hspace{0.15in} \= \hspace{0.15in} \= \hspace{0.75in} \\
{\bf while} (there are domain changes) \\
\> {\bf for} all $i$ from $0$ to $N-1$ \\
\>\> {\bf while} (shaving successful for $Domain(k_i)$) \\
\>\>\> add( $k_i = \max(Domain(k_i))$ ) \\
\>\>\> {\bf if} (!\textit{gMin}) \\
\>\>\>\> {\bf if} ( $\max(Domain(k_i))-1 \ge \min(Domain(k_i))$ ) \\
\>\>\>\>\> add( $k_i < \max(Domain(k_i))$ ) \\
\>\>\>\> {\bf else} \\
\>\>\>\>\> return $bestSolution$; stop, optimality has been proved \\
\>\>\> remove( $k_i = \max(Domain(k_i))$ ) \\
\end{tabbing}
\caption{\label{fig:wqShaving1}$W_q$-based shaving algorithm}
\end{figure}

The $W_q$-based shaving procedure makes inferences based strictly on the
constraint $W_q \le bestW_q$: the constraint $B \ge B_l$ is removed prior to running this procedure in order to eliminate the possibility of incorrect inferences. As in $B_l$-based shaving, a constraint of the form $k_i = \max(k_i)$, where $\max(k_i)$ is the maximum value in the domain of $k_i$, is added temporarily, and the function \textit{gMin} is used to assign values to the rest of the variables. Because the $B_l$ constraint has been removed, the only reason for the infeasibility of the policy is that it has a $W_q$ value greater than the best $W_q$ that has been encountered so far. Since all switching points except $k_i$ are assigned their smallest possible values, infeasibility implies that in any solution with a smaller expected waiting time, the value of $k_i$ has to be strictly smaller than $\max(k_i)$. 
This shaving procedure is stated in Figure \ref{fig:wqShaving1}.
%
%
%
\subsection{Combination of Shaving Procedures}
$W_q$-based and $B_l$-based shaving will result in different domain
reductions since they are based on two different constraints. Moreover,
using the two together may cause more domain modifications than when either
is used by itself. 
Therefore, it makes sense to run the $B_l$-based and $W_q$-based shaving procedures alternately (with $W_q$ and $B_l$ constraints added and removed appropriately) until no more domain pruning is possible. Such a combination of the two shaving procedures will be referred to as \textit{AlternatingShaving}.

The \textit{AlternatingShaving} procedure can be effectively combined with
search in the following manner. \textit{AlternatingShaving} can be run
initially, until no further domain modifications are possible. Search can
then be performed until a better solution is found, at which point
\textit{AlternatingShaving} can be applied again. Subsequently, search and
shaving can alternate until one of them proves optimality of the best
solution found. Such an approach may be successful because if search finds a
new best solution, a new constraint on $W_q$ will be added, and so
$W_q$-based shaving may be able to reduce the upper bounds of the switching
point variables. 
This way of combining search and shaving will be further referred to as \textit{AlternatingSearchAndShaving}.

Other variations of shaving are also possible. In particular, both $B_l$-based and $W_q$-based shaving procedures can be extended to make inferences about values of two switching points. 
For example, one can assign maximum values to a pair of switching point variables, while assigning minimum values to the rest. 
If the resulting policy is feasible, then a constraint stating that at least one variable from this pair has to be assigned 
a smaller value can be added to the problem. 
Preliminary experiments indicated that shaving procedures based 
on two switching points do not, in general, result in more effective models. Such procedures do not 
explicitly reduce the domains of the switching point variables but rather add a set of constraints to the model 
which do not appear, in practice, to significantly reduce the size of the search space. 
One possible direction for future work may be to further investigate these variations of shaving.
\section{Experimental Results\label{sec:expResults}}
Several sets of experiments\footnote{Numerical values in some of the results are slightly different from the ones presented in our previous work \cite{Terekhov07a} 
due to some minor errors discovered after the publication of that paper. 
The main conclusions and analysis of the previous work 
remain valid, however.} were performed in order to evaluate the efficiency
of the proposed models and the effectiveness of dominance rules and shaving
procedures, as well as to compare the performance of the best CP model with
the performance of heuristic $P_1$. 
All constraint programming models were implemented in ILOG Solver 6.2, while the heuristic of Berman et al. was implemented using C++. 

We note that numerical results obtained in the experiments are sensitive to the level of precision that is set. In all constraint programming models, we set the default ILOG Solver precision to 0.000001. 
%
Doing so implies that all floating point variables in our model are considered bound when the maximum ($max$) and minimum ($min$) values of their intervals are such that $((max - min)/(max\{1, |min|\}) \le 0.000001$ \cite{Solver62}. %
In order to propagate constraints involving floating point variables, such
as Equation (\ref{eqn:switchProb}), ILOG Solver uses standard interval
arithmetic and outward rounding.\footnote{Jean-Fran\c cois Puget - personal communication.} 
 
\subsection{Problem Instances}

The information gained from policies $\hat K$ and ${\hat {\hat K}}$ is
explicitly used in the implementation of all three models. 
If $\hat {\hat K}$ is infeasible, then the program stops as there is no feasible solution for that instance. Otherwise, if $\hat K$ is feasible, then it is optimal. The two cases when $\hat K$ is optimal or $\hat {\hat K}$ is infeasible are therefore trivial and are solved easily both by the CP models and by Berman et al.'s heuristic. Although instances in which $\hat {\hat K}$ is optimal are very hard to solve without shaving, using the elementary $B_l$-based shaving procedure will always result in a (usually fast) proof of the optimality of this policy. This case is also trivial for Berman et al.'s heuristic. Consequently, the experimental results presented here are based only on the instances for which the optimal solution is between ${\hat K}$ and ${\hat {\hat K}}$.

Preliminary experiments indicated that the value of $S$ has a significant
impact on the efficiency of the programs since higher values of $S$ result
in larger domains for the $k_i$ variables for all models and a higher number
of $w_j$ variables for the \textit{Dual} model. As indicated in Table
\ref{tab:modelSummary} in Section \ref{sec:modelSummary1}, $S$ also has a
big impact on the number of constraints in the \textit{If-Then} and
\textit{Dual} models. Therefore, we consider instances for each value of $S$ from the set $\{10, 20, \dots, 100\}$ in order to gain an accurate understanding of the performance of the model and the heuristic. We note that for most instances with $S$ greater than 100, neither our method
nor Berman et al.'s heuristic $P_1$ may be used due to numerical
instability.

Thirty instances were generated for each $S$ in such a way as to ensure that
the instance is feasible and that the optimal policy is neither ${\hat{\hat
K}}$ nor ${\hat K}$. We created random combinations of parameter values for
each $S$ and chose the instances for which policy ${\hat{\hat K}}$ was found
to be feasible, but not optimal, and ${\hat K}$ was determined to be
infeasible. In order to check that ${\hat{\hat K}}$ is not optimal, it is
sufficient to find a feasible solution that has one switching point assigned
a value lower than its upper bound. When generating combinations of
parameters, the values of $N$ were chosen with uniform probability from the
set \{2, \dots, 38\}, the values of $\lambda$ from \{5,\dots, 99\}, the
values of $\mu$ from \{1, \dots, 49\} and the values of $B_l$ from \{1,
\dots, 4\}. 
There appears to be no easy way of determining whether a given instance will have ${\hat{\hat K}}$ or ${\hat K}$ as the optimal solution based only on the parameter values. Moreover, preliminary experiments indicated that problem difficulty depends on a combination of problem parameters (especially $S$, $N$ and $B_l$) rather than on one parameter only.  

A 10-minute time limit on the overall run-time of the program was enforced in the experiments. All experiments were performed on a Dual Core AMD 270 CPU with 1 MB cache, 4 GB of main memory, running Red Hat Enterprise Linux 4.

\subsection{Performance Measures}
In order to perform comparisons among the CP models, and between the CP models and Berman et al.'s heuristic, we look at mean run-times, the number of instances in which the optimal solution was found, the number of instances in which optimality was proved, the number of instances in which the best-known solution was found and the mean relative error (MRE). MRE is a measure of solution quality that allows one to observe
how quickly a particular algorithm is able to find a good solution. MRE is
defined as
\begin{eqnarray}
MRE(a,M) &=&  \frac{1}{\mid M \mid} 
\sum_{m \in M}^{} \frac{c(a,m)-c^*(m)}{c^*(m)} \label{eqn:MRE} 
\end{eqnarray}
where $a$ is the algorithm that is used to solve the problem, $M$
is the set of problem instances on which the algorithm is being tested,
$c(a, m)$ is the cost of a solution found for instance $m$ by algorithm $a$,
and $c^*(m)$ is the best-known solution for instance $m$. As we generated
the instances, $c^*(m)$ is the best solution we found during our these
experiments. 
\subsection{Comparison of Constraint Programming Models and Techniques}
Each CP model was tested with and without shaving and dominance rules. A
total of 30 CP-based methods were therefore evaluated. A model with $B_l$-based shaving is a model which runs the
$B_l$-based shaving procedure until no more domain changes are possible,
adds a constraint on the value of $W_q$ based on the best solution found
during the shaving procedure and runs search for the rest of the
time. Similarly, models with $W_q$-based shaving and
\textit{AlternatingShaving} are models which run the $W_q$-based shaving
procedure and the \textit{AlternatingShaving} procedure, respectively, until
it is no longer possible to reduce the domains of the switching point
variables, add a constraint requiring $W_q$ to be less than the expected
waiting time of the best solution found by the shaving procedure and use
search for the rest of the time. As described previously,
\textit{AlternatingSearchAndShaving} alternates between search and the
\textit{AlternatingShaving} procedure. In all models, search assigns
switching points in increasing index order. The smallest value in the domain
of each variable is tried first. 

%
%
%
%
\subsubsection{Constraint Programming Models}
Table \ref{tab:summ_res1} presents the number of instances, out of 300, for which the optimal solution was found and proved by each of the 30 proposed CP-based methods. This table indicates that the \textit{PSums} model is the most effective of the three, proving optimality in the largest number of instances regardless of the use of dominance rules and shaving. With \textit{AlternatingSearchAndShaving}, \textit{PSums} proves optimality in the largest number of instances: in $79.3\%$ of all instances, 238 out of 239 instances for which optimality has been proved by any model.

Figure \ref{fig:cpModelsHeuristicMRE} shows how the MRE changes over the
first 50 seconds of run-time for \textit{If-Then}, \textit{PSums} and
\textit{Dual} models with \textit{AlternatingSearchAndShaving}, and for
Berman's heuristic (we comment on the performance of the heuristic in
Section \ref{sec:heuristicPerformance}). \textit{PSums} is, on average, able
to find better solutions than the other two models given the same amount of
run-time.

\begin{table}
\begin{center}
\begin{tabular}{|l|c|c|c|c|c|c|c|c|c|c|}  \hline
& \multicolumn{2}{|c|}{No Shaving} 
& \multicolumn{2}{|c|}{$B_l$-based} 
& \multicolumn{2}{|c|}{$W_q$-based}
& \multicolumn{2}{|c|}{$Alternating$}
& \multicolumn{2}{|c|}{$AlternatingSearch$} \\
& \multicolumn{2}{|c|}{}
& \multicolumn{2}{|c|}{Shaving} 
& \multicolumn{2}{|c|}{Shaving}
& \multicolumn{2}{|c|}{$Shaving$}
& \multicolumn{2}{|c|}{$AndShaving$} \\ \hline 
{} & D & ND & D & ND & D & ND & D & ND & D & ND  \\ \hline \hline
\textit{If-Then} & 105 & 105 & 192 & 191 & 105 & 105 & 219 & 218 & 234 & 234 \\ \hline
\textit{PSums} & 126 & 126 & 202 & 201 & 126 & 126 & 225 & 225 & 238 & 238 \\ \hline
\textit{Dual} & 105 & 105 & 191 & 191 & 105 & 105 & 218 & 218 & 232 & 232 \\ \hline
\end{tabular}
\caption{\label{tab:summ_res1}Number of instances for which the optimal solution is found and its optimality is proved within 10 CPU-minutes out of a total of 300 problem instances (D - with dominance rules, ND - without dominance rules).}
\end{center}
\end{table}

\begin{figure}[t]
\begin{center}
\centerline{\rotatebox{0}{\scalebox{0.9}{\epsfig{file=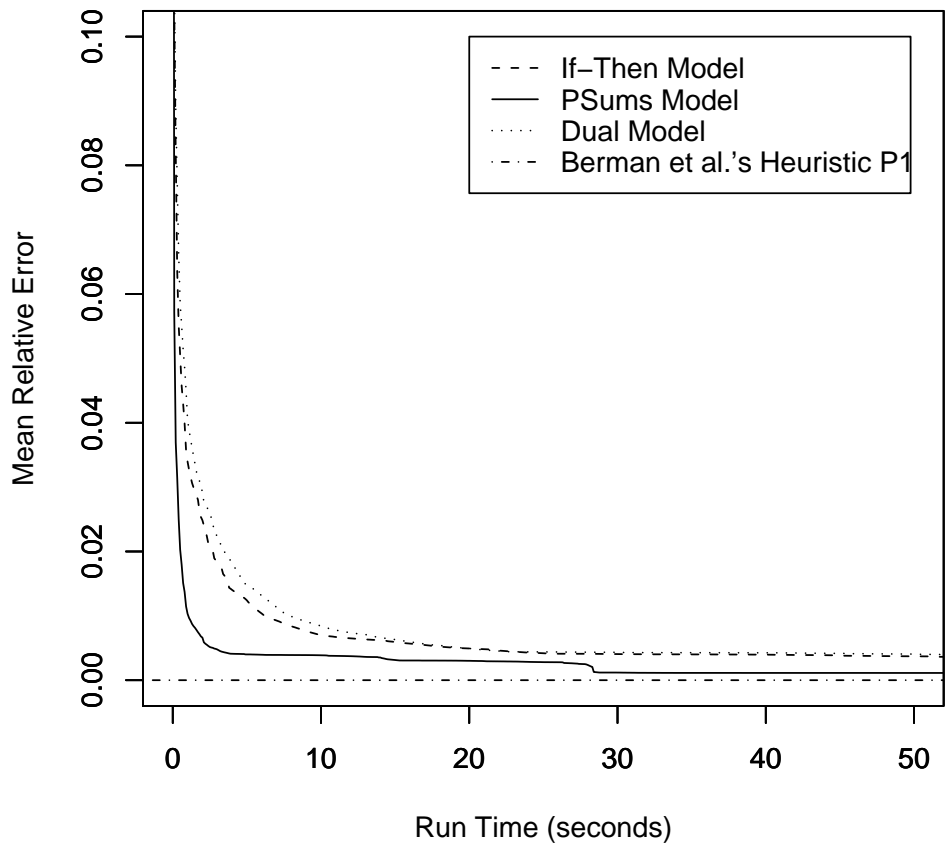}}}}
\end{center}
\caption{\label{fig:cpModelsHeuristicMRE}Comparison of MRE of three CP models with
\textit{AlternatingSearchAndShaving} with Berman's heuristic $P_1$.}
\end{figure}

%
%
%
%
%
%
%
%
%
%
%
%
In Table \ref{tab:summ_res2}, additional statistics regarding the performance of the three models with \textit{AlternatingSearch}\textit{AndShaving} and without dominance rules are presented (we comment on the same statistics for $P_1$ in Section \ref{sec:heuristicPerformance}). In particular, for each model, the number of instances in which it finds the best solution (out of 300), the number of instances in which it finds the optimal solution (out of 239 cases for which optimality has been proved) and the number of times it proves optimality (out of 300) are presented.
It can
be seen that all models find the optimal solution in 
the 239 instances for which it is known. 
However, the \textit{PSums} model 
proves optimality in 4 more instances than the \textit{If-Then} model and in 6 more instances than the \textit{Dual}. \textit{PSums} also finds the 
best-known solution of any algorithm in $97.6\%$ of all the instances considered. 
A more detailed discussion of the differences in the performance of the CP models is presented in Section \ref{sec:cpModelsDiff}.
%
%
%
\subsubsection{Shaving Procedures}
%
%
%
%
%
%
%
From Table \ref{tab:summ_res1}, it can be observed that the CP models without shaving and with the $W_q$-based shaving procedure prove optimality in the fewest number of cases. 
The similarity in performance of models without shaving and with $W_q$-based shaving is not surprising because the $W_q$-based procedure is able to start pruning 
domains only when the value of the best policy found prior to this procedure is quite good. 
When the $W_q$-based procedure is used alone, it only has one solution to base its inferences on, 
namely $\hat {\hat K}$. Since all policies will result in a smaller expected waiting time than $\hat {\hat K}$, this procedure by itself is useless. 

Employing the $B_l$-based shaving procedure substantially improves the performance of all models: without dominance rules, the \textit{If-Then}, the \textit{PSums} and the \textit{Dual} models prove optimality in 86, 75 and 86 more instances, respectively, than the corresponding models without shaving or with only $W_q$-based shaving; with dominance rules, the situation is equivalent.
These results imply that inferences made based on the $B_l$ constraint are effective in reducing the domains of the decision variables.

Models with \textit{AlternatingShaving} and \textit{AlternatingSearchAndShaving} perform even better 
than models employing only the $B_l$-based shaving procedure. The real 
power of $W_q$-based shaving only becomes apparent when it is combined with
$B_l$-based shaving because $B_l$-based shaving often finds a good solution and a good value of $bestW_q$, allowing 
the $W_q$-based procedure to infer more domain reductions. This observation explains 
why \textit{AlternatingSearchAndShaving} performs better than \textit{AlternatingShaving}. 
In particular, in \textit{AlternatingSearchAndShaving}, the $W_q$-based procedure is used both after a new 
best solution is found during shaving and after one is found during search. Therefore, if a higher quality 
solution is found during search, it will be used by the $W_q$-based procedure to further prune the domains of the 
switching point variables. 

In Figure \ref{fig:shavingMRE}, the average run-times for each value of $S$
from 10 to 100 are presented for the four shaving procedures with the
\textit{PSums} model. Since a run-time limit of 600 seconds was used
throughout the experiments, we assumed a run-time of 600 seconds for all
instances for which optimality has not been proved within this
limit. Therefore, the mean run-times reported throughout this paper are
underestimates of the true means. Figure \ref{fig:shavingMRE} shows that,
for each value of $S$, the \textit{AlternatingSearchAndShaving} procedure
gives the best performance. It can also be seen that, as $S$ increases, it
becomes increasingly difficult to prove optimality and average run-times
increase. As stated previously, this is due to larger domains of the
switching point variables. The \textit{AlternatingSearchAndShaving}
procedure, however, is able to significantly reduce the domains of the $k_i$
variables and therefore provides an effective method for instances with
higher values of $S$ as well.

\begin{table}
\begin{center}
\begin{tabular}{|l|c|c|c|} \hline
{} & \# best found (/300) & \# opt. found (/239) & \# opt. proved (/300) \\ \hline \hline
\textit{PSums} & 293 & 239 & 238 \\ \hline
\textit{If-Then} & 282 & 239 & 234 \\ \hline
\textit{Dual} & 279 & 239 & 232 \\ \hline
$P_1$ & 282 & 238 & 0 \\ \hline
\end{tabular}
\caption{\label{tab:summ_res2}Comparison of three CP models (with
\textit{AlternatingSearchAndShaving} and without dominance rules) with Berman's heuristic $P_1$.}
\end{center}
\end{table}

\begin{figure}[t]
\begin{center}
\centerline{\rotatebox{0}{\scalebox{0.7}{\epsfig{file=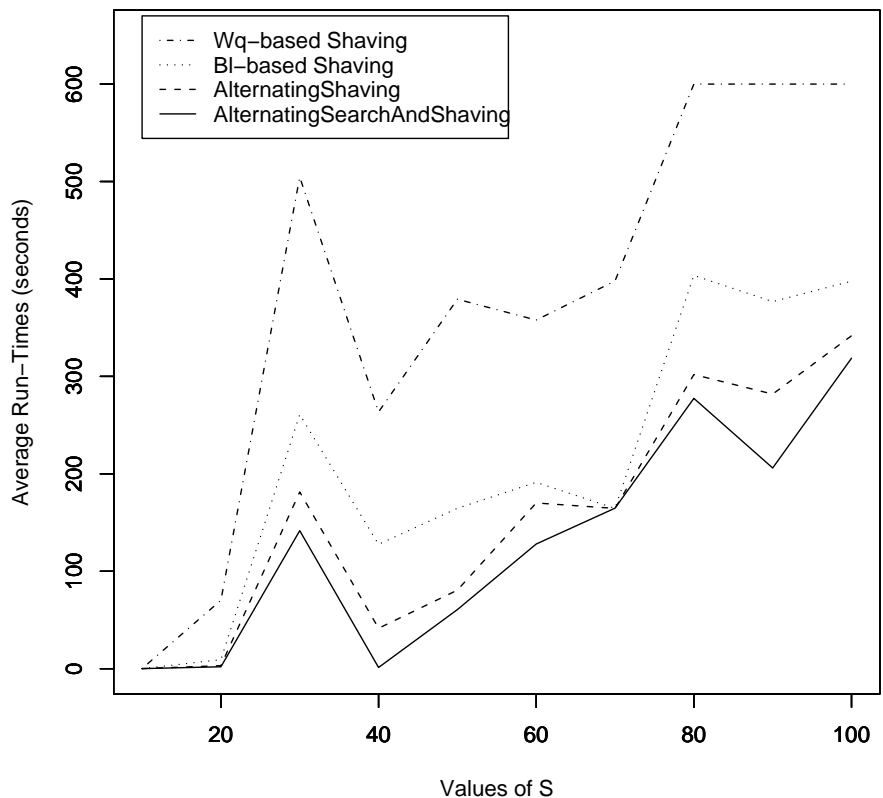}}}}
\end{center}
\caption{\label{fig:shavingMRE}\textit{PSums} model with various shaving techniques: average run-times for each value of $S$. Average run-times for \textit{PSums} without shaving are not shown in this graph since the resulting curve would be indistinguishable from the one for $W_q$-based shaving.}
\end{figure}
\subsubsection{Dominance Rules} 
Table \ref{tab:summ_res1} indicates that there is rarely any difference in
the number of instances solved to optimality between models with and without
dominance rules. No difference at all is visible for any model without
shaving, with $W_q$-based shaving and \textit{AlternatingSearchAndShaving}.

Recall that dominance rules are implemented by the addition of a constraint on the values of the switching point variables after a solution is found. Such a constraint will be more effective when more of the switching point variables are assigned their minimum values in the current solution. Usually, such policies are also the ones which result in a smaller expected waiting time. Similarly, $W_q$-based shaving is useful only when a solution with small expected waiting time is found. This leads to the conjecture that dominance rules may be effective only for the same instances for which the $W_q$-based shaving procedure is effective. This conjecture is supported by the results of Table \ref{tab:summ_res1}. In particular, when the $W_q$-based shaving procedure is used by itself, it makes inferences based only on the policy $\hat{\hat K}$, a solution that is generally of poorest quality for an instance. The method with a single run of $W_q$-based shaving therefore heavily relies on search.  
Since search takes a long time to find a feasible solution of good quality, the effectiveness of dominance rule-based constraints is also not visible within the given time limit.

On the other hand, in \textit{AlternatingSearchAndShaving}, the $W_q$-based
procedure plays a key role because it makes domain reductions based on high
quality solutions produced by $B_l$-based shaving, and later,
search. Dominance rules do not play a role in this procedure since shaving
is used after every new solution is found. However, even if dominance rule constraints were explicitly incorporated in the procedure (i.e. if they were added before each new run of search), they would be redundant since they serve essentially the same purpose as the $W_q$-based shaving procedure.

When shaving is not used, the results are equivalent to those achieved when
$W_q$-based shaving is employed. The explanation for the absence of a
difference between models with and without dominance rules is therefore also
the same. In particular, it takes a long time for a solution to be found
whose quality is such that it allows the dominance rule constraint to
effectively reduce the size of the search tree. 
 
When $B_l$-based shaving and \textit{AlternatingShaving} are used, dominance rules are sometimes helpful. In both cases, this is because after these two shaving procedures, subsequent search usually finds a good solution quickly, and, since $W_q$-based shaving is not used again at that point, the dominance rule constraint that is added can be effective in reducing the size of the search tree. 
\begin{figure}[t]
\begin{center}
\centerline{\rotatebox{0}{\scalebox{0.9}{\epsfig{file=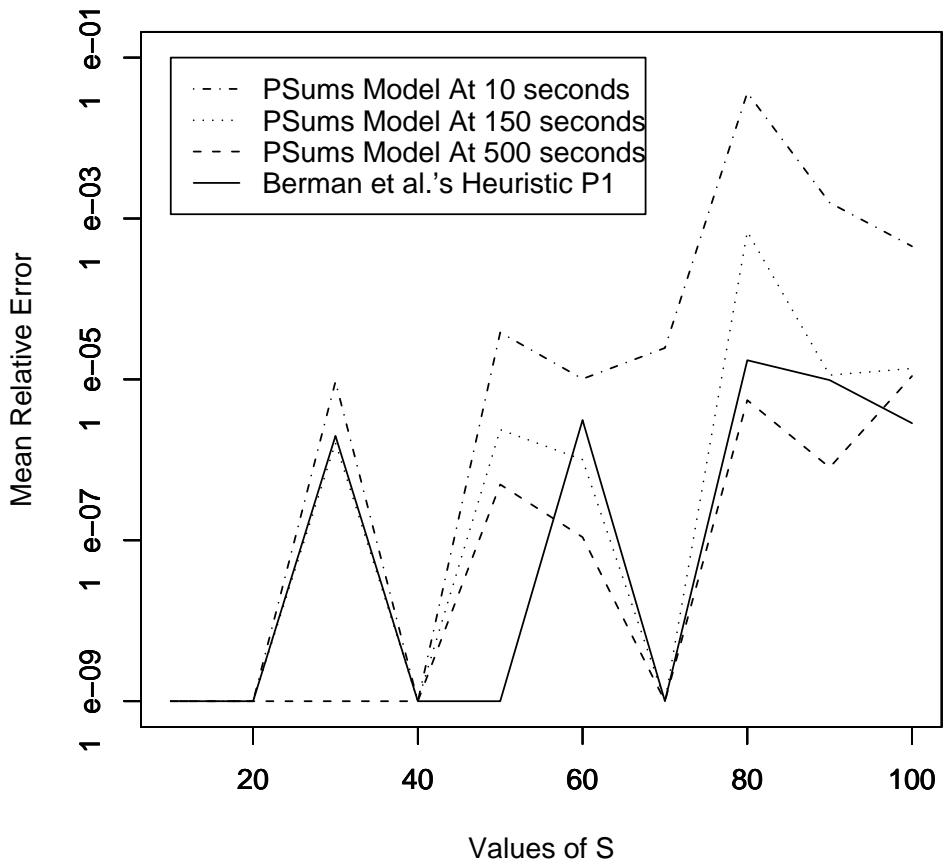}}}}
\end{center}
\caption{\label{fig:MREforeachS}MRE for each value of $S$ for $P_1$ and the $PSums$ model.} 
\end{figure}

Overall, it can be observed that using \textit{AlternatingSearchAndShaving} without dominance rules is more effective than using $B_l$-based shaving or \textit{AlternatingShaving} with dominance rules. Therefore, in further comparisons, the focus is only on models with \newline \textit{AlternatingSearchAndShaving} without dominance rules.
\subsection{Heuristic $P_1$ vs. Best Constraint Programming Approach \label{sec:heuristicPerformance}}
Empirical results regarding the performance of heuristic $P_1$ are not
presented by Berman et al. \citeyear{Berman05a}, and so the ability of $P_1$ to find
good switching policies has not been explicitly evaluated in previous work. We wanted to find out 
how well the heuristic actually performs by comparing it to our CP methods.

In Table \ref{tab:summ_res2}, we present several measures of performance for the three proposed models with \textit{AlternatingSearchAndShaving} and for the heuristic $P_1$.  
The heuristic performs well, finding the best-known solution in only eleven fewer instances than the \textit{PSums} model, in three more instances than the \textit{Dual} model and in the same number of instances as the \textit{If-Then} model. Moreover, the heuristic finds, but, of course, cannot prove, the optimal solution in 238 out of 239 instances for which the optimal is known. The three CP models find the optimal solution in all 239 of these.
The run-time of the heuristic is negligible, whereas the mean run-time of the \textit{PSums} model is approximately 130 seconds (the mean run-times of the other two models are slightly higher: 141 seconds for the \textit{If-Then} model and 149 seconds for the \textit{Dual} model).

Table \ref{tab:summ_res2} also shows that the \textit{PSums} model is able to find the best-known solution in 11 more instances than the heuristic. Closer examination reveals that there are 275 instances in which both the \textit{PSums} model and $P_1$ find the best-known solution, 18 instances in which only \textit{PSums} is able to do so and 7 in which only the heuristic finds the best-known.

From Figure \ref{fig:cpModelsHeuristicMRE}, it can be observed that the heuristic achieves a very small MRE in a negligible amount of time. After about 50 seconds of run-time, the MRE over 300 instances resulting from \textit{PSums} with \textit{AlternatingSearchAndShaving} becomes comparable to that of the heuristic MRE. In Figure \ref{fig:MREforeachS}, the MRE over 30 instances for each value of $S$ is presented for the heuristic and for \textit{PSums} with \textit{AlternatingSearchAndShaving} at 10, 150 and 500 seconds of run-time. After 10 seconds, the performance of \textit{PSums} is comparable to that of the heuristic for values of $S$ smaller than or equal to 40, but the heuristic appears to be quite a bit better for higher values of $S$. At 150 seconds, the performance of \textit{PSums} is comparable to that of the heuristic except at $S$ values of 50 and 80. After 500 seconds, \textit{PSums} has a smaller MRE over the 300 instances and also a lower (or equal) MRE for each value of $S$ except 50 and 100.

Overall, these results indicate that the heuristic performs well--its run-time is negligible, it finds the optimal solution in all but one of the cases for which it is known, and it finds the best solution in $94\%$ of all instances. Moreover, it results in very low MRE. Although \textit{PSums} with \textit{AlternatingSearchAndShaving} is able to achieve slightly higher numbers in most of the performance measures, it is clear that these improvements are small given that the \textit{PSums} run-time is so much higher than the run-time of the heuristic.
%
\section{\textit{PSums}-$P_1$ Hybrid}
\label{sec:hybrid} 
Naturally, it is desirable to create a method that would be able to find a solution of high quality in a very short amount of time, as does Berman's heuristic, and that would also have the same high rate of being able to prove optimality within a reasonable run-time as does \textit{PSums} with \textit{AlternatingSearchAndShaving}. It is therefore worthwhile to experiment with a \textit{PSums}-$P_1$ Hybrid, which starts off by running $P_1$ and then, assuming the instance is feasible, uses the \textit{PSums} model with \textit{AlternatingSearchAndShaving} to find a better solution or prove the optimality of the solution found by $P_1$ (infeasibility of an instance is proved if the heuristic determines that policy ${\hat {\hat K}}$ is infeasible).
%
%
%
%

Since it was shown that heuristic $P_1$ is very fast, running it first
incurs almost no overhead. Throughout the analysis of experimental results,
it was also noted that the performance of the $W_q$-based shaving procedure
depends on the quality of the best solution found before it is used. We have
shown that the heuristic provides solutions of very high quality. Therefore,
the first iteration of the $W_q$-based procedure may be able to significantly prune the domains of switching point variables because of the good-quality solution found by the heuristic. Continuing by alternating the two shaving techniques and search, which has also been shown to be an effective approach, should lead to good results.

\begin{table}
\begin{center}
\begin{tabular}{|l|c|c|c|} \hline
{} & \# best found (/300) & \# opt. found (/239) & \# opt. proved (/300) \\ \hline \hline
\textit{PSums} & 293 & 239 & 238 \\ \hline
$P_1$ & 282 & 238 & 0 \\ \hline
\textit{PSums}-$P_1$ & 300 & 239 & 238 \\ 
Hybrid              &     &     &     \\ \hline
\end{tabular}
\caption{\label{tab:summ_res3}Comparison of \textit{PSums} model with \textit{AlternatingSearchAndShaving} 
and Berman et al.'s heuristic $P_1$ with the Hybrid model.}
\end{center}
\end{table}

The proposed Hybrid algorithm was tested on the same set of $300$ instances that was used above. Results illustrating the performance of the Hybrid as well as the performance of $P_1$ and \textit{PSums} with \textit{AlternatingSearchAndShaving} are presented in Table \ref{tab:summ_res3}. The Hybrid is able to find the best solution in all $300$ cases: in the 275 instances in which both the heuristic and \textit{PSums} find the best-known solution, in 18 in which only \textit{PSums} finds the best-known and in 7 in which only the heuristic does so. The Hybrid finds the optimal solution (for those instances for which it is known) and proves optimality in as many instances as the \textit{PSums} model.
The mean run-time for the Hybrid is essentially identical to the mean run-time of \textit{PSums} with \textit{AlternatingSearchAndShaving}, equalling approximately 130 seconds.

Thus, the Hybrid is the best choice for solving this problem: it finds as good a solution as the heuristic in as little time (close to $0$ seconds), it is able to prove optimality in as many instances as the best constraint programming method, and it finds the best-known solution in all instances considered. Moreover, all these improvements are achieved without an increase in the average run-time over the \textit{PSums} model.

\section{\label{sec:discussion}Discussion}

In this section, we examine some of the reasons for the poor performance of the CP models without shaving, suggest reasons for the observed differences among the CP models, discuss the performance of the Hybrid and present some perspectives on our work.
 
\subsection{Lack of Back-Propagation \label{sec:lackofpropagation}}
In our experiments, we have some instances for which even the
\textit{PSums}-$P_1$ Hybrid with \textit{AlternatingSearchAndShaving} is
unable to find and prove the optimal solution within the $10$-minute time
limit. In fact, in many of these instances, the amount of time spent during search is 
higher than the time spent on shaving, and the run-time limit is usually 
reached during the search, rather than the shaving, phase. Further analysis of the algorithms' behaviour suggests that this poor performance of search can be explained by the lack of
back-propagation. Back-propagation refers to the pruning of the domains of
the decision variables due to the addition of a constraint on the objective
function: the objective constraint propagates ``back'' to the decision
variables, removing domain values and so reducing search. In the CP models
presented above, there is very little back-propagation.

Consider a model without shaving. Throughout search, if a
new best solution is found, the constraint $W_q \le bestW_q$, where
$bestW_q$ is the new objective value, is added to the model. However, the
domains of the switching point variables are usually not reduced in any way
after the addition of such a constraint. This can be illustrated by
observing the amount of propagation that occurs in the \textit{PSums} model when $W_q$ is
constrained. 

For example, consider an instance of the problem with $S = 6$, $N = 3$,
$\lambda = 15$, $\mu = 3$, and $B_l = 0.32$ (this instance is used in Section \ref{sec:shaving} to illustrate 
the shaving procedures). The initial domains of the
switching point variables are $[0..3]$, $[1..4]$, $[2..5]$ and $[6]$. The
initial domains of the probability variables $P(k_i)$ for each $i$, after
the addition of $W_q$ bounds provided by ${\hat {K}}$ and $\hat {\hat {K}}$,
are listed in Table \ref{tab:domains}. The initial domain of $W_q$, also
determined by the objective function values of $\hat {K}$ and $\hat {\hat
{K}}$, is $[0.22225..0.425225]$. The initial domains of $L$ and $F$, are
$[2.8175e^{-7}..6]$ and $[0..2.68]$, respectively. Upon the addition of the
constraint $W_q \le 0.306323$, where $0.306323$ is the known optimal value
for this instance, the domain of $W_q$ is reduced to $[0.22225..0.306323]$,
the domain of $L$ becomes $[1.68024..6]$ and the domain of $F$ remains
$[0..2.68]$. The domains of $P(k_i)$ after this addition are listed in Table
\ref{tab:domains}. The domains of both types of probability variables are
reduced by the addition of the new $W_q$ constraint. However, the domains of
the switching point variables remain unchanged. Therefore, even though all
policies with value of $W_q$ less than $0.306323$ are infeasible,
constraining $W_q$ to be less than or equal to this value does not result in any
reduction of the search space. It is still necessary to search through all policies in order to show that no better feasible solution exists.

\begin{table}
\begin{center}
\begin{tabular}{|l|c|c|c|c|} \hline
& \multicolumn{2}{|c|}{Before addition of $W_q \le 0.306323$} & \multicolumn{2}{|c|}
{After addition of $W_q \le 0.306323$}\\ \hline \hline
$j$ & $P(j)$ & $PSums(j)$ & $P(j)$ & $PSums(j)$ \\ \hline
$k_0$ & $[4.40235e^{-6}..0.979592]$ & $[0..1]$ & $[4.40235e^{-6}..0.979592]$ & 
$[0..0.683666]$ \\ \hline
$k_1$ & $[1.76094e^{-7}..1]$ & $[0..1]$ & $[0.000929106..1]$ & 
$[0..0.683666]$ \\ \hline
$k_2$ & $[2.8175e^{-8}..0.6]$ & $[2.8175e^{-8}..1]$ & $[0.0362932..0.578224]$ & 
$[0.0362932..0.71996]$ \\ \hline
$k_3$ & $[4.6958e^{-8}..1]$ & N/A & $[0.28004..0.963707]$ & 
N/A \\ \hline
\end{tabular}
\caption{\label{tab:domains}Domains of $P(j)$ and $PSums(j)$ variables for $j = k_0, k_1, k_2,
k_3$, before and after the addition of the constraint $W_q \le 0.306323$.}
\end{center}
\end{table}

One of the reasons for the lack of pruning of the domains of the $k_i$
variables due to the $W_q$ constraint is likely the complexity of the
expression $W_q = \frac{L}{\lambda (1 - P(k_N))} - \frac{1}{\mu}$. In the
example above, when $W_q$ is constrained to be less than or equal to
0.306323, we get the constraint $0.306323 \ge \frac{L}{15(1 - P(k_N))} -
\frac{1}{3}$ , which implies that $9.594845(1-P(k_N)) \ge L$. This explains
why the domains of both $L$ and $P(k_N)$ change upon this addition to the
model. The domains of the rest of the $P(k_i)$ variables change because of
the relationships between $P(k_i)$'s (Equation (\ref{eqn:switchProb})) and
because of the constraint that the sum of all probability variables has to
be 1. Similarly, the domains of $PSums(k_i)$'s change because these variables
are expressed in terms of $P(k_i)$ (Equation (\ref{eqn:pSums})). However, because the actual $k_i$
variables mostly occur as exponents in expressions for $PSums(k_i)$,
$P(k_i)$, and $L(k_i)$, the minor changes in the domains of $PSums(k_i)$,
$P(k_i)$, or $L(k_i)$ that happen due to the constraint on $W_q$ have no
effect on the domains of the $k_i$. This analysis suggests that it may be
interesting to investigate a CP model based on log-probabilities rather than
on the probabilities themselves. Such a model may lead to stronger
propagation.

Likewise, in the \textit{If-Then} and \textit{Dual} models, the domains of
the decision variables are not reduced when a bound on the objective
function value is added, although the domains of all probabilities, $L$ and
$F$ are modified. In both of these models, the constraints relating $F$, $L$ and the 
probability variables to the variables $k_i$ are the balance equations, which are quite 
complex. The domains of the probability variables do not
seem to be reduced significantly enough due to the new $W_q$ bound so as to
result in the pruning of $k_i$ domains because of these constraints. 

These observations served as the motivation for the proposed shaving
techniques. In particular, the $W_q$-based shaving procedure reduces the
domains of switching point variables when it can be shown that these values
will necessarily result in a higher $W_q$ value than the best one found up
to that point. This makes up for some of the lack of
back-propagation. However, even when this procedure is used after each new
best solution is found, as in \textit{AlternatingSearchAndShaving}, it is
not always able to prune enough values from the domains of the $k_i$'s so as
to be able to prove optimality within 10 minutes of run-time. It can
therefore be seen that inferences based on the value of $W_q$ are very
limited in their power and, therefore, if the domains of switching point
variables are large after shaving, then it will not be possible to prove
optimality in a short period of time.

\subsection{Differences in the Constraint Programming Models \label{sec:cpModelsDiff}}
\begin{table}[t]
\begin{center}
\begin{tabular}{|l|c|c|c|c|} \hline
{} 
& \multicolumn{2}{|c|}{21 Instances Solved by \textit{PSums} Only}
& \multicolumn{2}{|c|}{105 Instances Solved by All Models} \\ \hline
$$ & First Solution & Total & First Solution & Total \\ \hline \hline
\textit{If-Then} & 8234 & 137592 & 1201 & 37596  \\ \hline
\textit{PSums} & 6415 & 464928 & 1064 & 36590 \\ \hline
\textit{Dual} & 7528 & 102408 & 1132 & 36842 \\ \hline
\end{tabular} 
\caption{\label{tab:choicePoints}Mean number of choice points explored before the first 
solution is found and mean total number of choice points explored within 600 seconds for the three 
models without shaving and without dominance rules. For \textit{PSums}, 
the latter statistic corresponds to the total number of choice points needed to prove optimality.} 
\end{center}
\end{table}

%
\begin{figure}[t]
\begin{center}
\centerline{\rotatebox{0}{\scalebox{0.9}{\epsfig{file=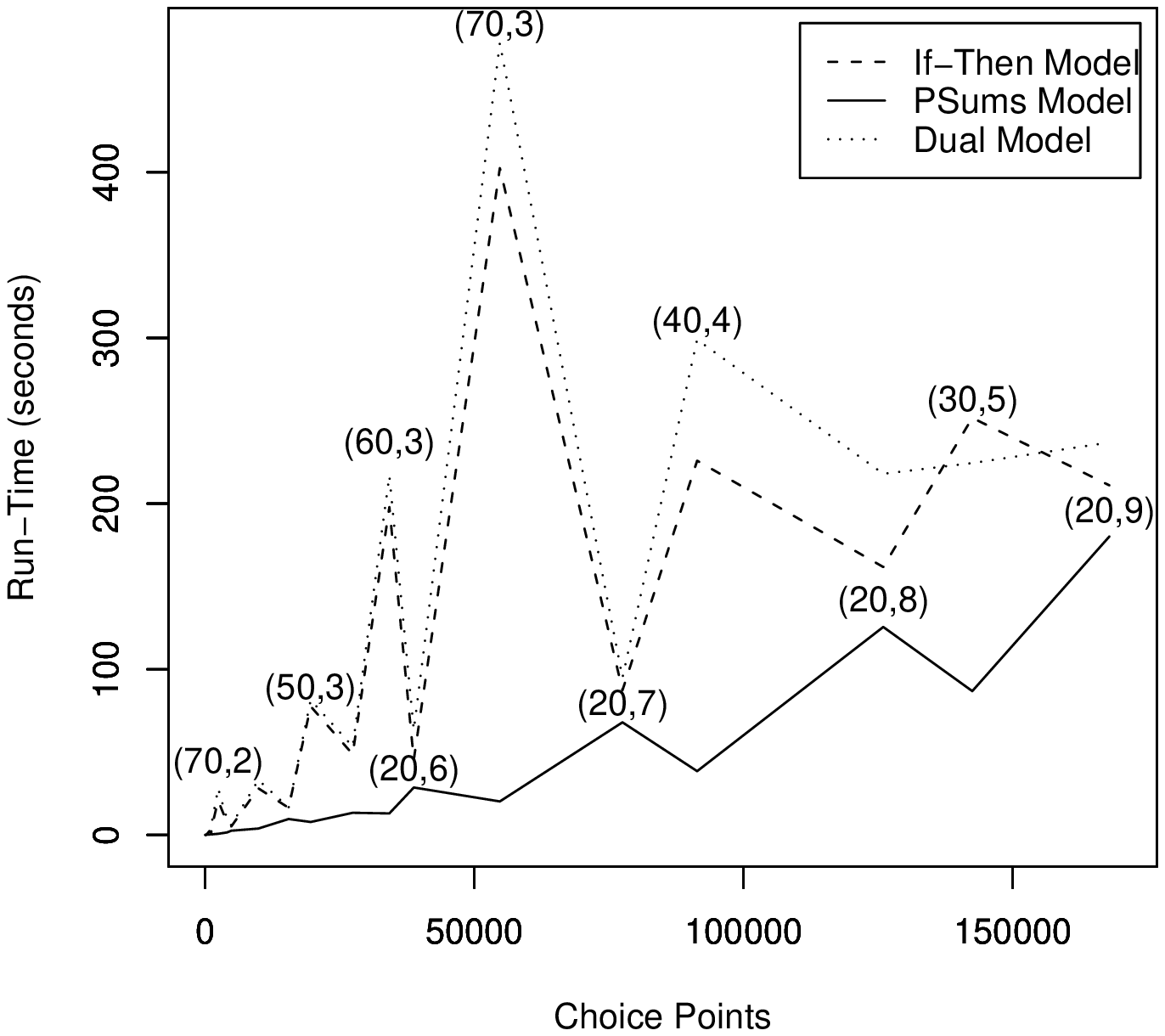}}}}
\end{center}
\caption{\label{fig:choicePointsByProblemSize}Run-times averaged over all
instances with equal number of choice points explored, for 82 instances for
which the number of choice points is the same for all models. The labels in
the points indicate the ($S$, $N$) values for the corresponding instances. }
\end{figure}

Experimental results demonstrate that the \textit{PSums} model is the best out of the three models both with and without shaving. In this section, we examine the models in more detail in an attempt to understand the reasons for such differences. 

\subsubsection{Comparison of \textit{PSums} and the other two models}
In order to analyze the performance of the models without shaving, we look at the mean number of choice points statistics, which give an indication of the size of the search space that is explored. 
To compare our three models, we look at the mean number of choice points considered before the first feasible solution is found and the mean total number of choice points explored within 600 seconds of run-time. 

In Table \ref{tab:summ_res1}, it is shown that, without shaving, the \textit{PSums} model proves optimality in 21 more instances than the other two models and that there are 105 instances in which all three models prove optimality. In Table \ref{tab:choicePoints}, we present the mean number of choice points statistics for all three models for both of these sets of instances. It can be seen that the mean number of choice points that need to be explored by the \textit{PSums} model in order to find an initial solution is smaller than for the other two models, both for the 105 instances that are eventually solved to optimality by all models and for the 21 instances that are only solved to optimality by \textit{PSums}. Because the same variable and value ordering heuristics are used in all models, this observation implies that more propagation occurs during search in the \textit{PSums} model than in the other two models. This claim is further supported by the fact that the mean total number of choice points for the 105 instances that are solved by all models is smaller for \textit{PSums} than for the other two models.

Table \ref{tab:choicePoints} also shows that, for the 21 instances that are only solved to optimality by \textit{PSums}, the mean total number of choice points is the highest for the \textit{PSums} model. Since \textit{PSums} is the only one out of the three models to solve these instances, this implies that propagation is happening faster in this model. This observation is confirmed by the results from the 105 instances that are solved by all three models: for these instances, the \textit{Dual} explores an average of 713 choice points per second, the \textit{If-Then} model explores an average of 895 choice points per second and the \textit{PSums} model explores an average of 1989 choice points per second. In other words, it appears that propagation in the \textit{PSums} model is more than twice as fast as in the other two models. 

A more detailed examination of the results showed that in 82 out of the 105
instances that were solved by all models, the number of choice points
explored, for a given instance, is the same for all models. Moreover, for
instances which have the same value of $S$ and $N$, the number of choice
points explored is equal. In Figure \ref{fig:choicePointsByProblemSize}, the
run-times of the three models as the number of choice points
increases is illustrated. In order to create this graph, we averaged the
run-times of all instances for which the number of choice points examined is
the same. Some points in the figure are labeled as $(S, N)$ in order to show
the relationship between the number of choice points, the values of $S$ and
$N$, and the run-times. We note that there is one instance, when $S = 10$
and $N = 6$, for which the number of choice points is the same as for the
instances when $S = 10$ and $N = 4$. However, for all other instances (out
of 82), there is a one-to-one correspondence between $(S, N)$ and the number
of choice points.

Several observations can be made from Figure \ref{fig:choicePointsByProblemSize}. Firstly, this graph demonstrates that propagation in the \textit{PSums} model is faster than in the other models.
Secondly, the behaviour of the \textit{PSums} model appears to be quite different from that of both the \textit{If-Then} and the \textit{Dual} models. The run-times of the \textit{PSums} model seem to be significantly influenced by the value of $N$. For example, when $S = 20$, the run-times of this model increase as $N$ increases from 6 to 9. Moreover, given two instances of which one has a high $S$ and a low $N$, and the other has a low $S$ and a high $N$, the \textit{PSums} model generally needs a longer time to solve the instance with a low $S$ and a high $N$ (e.g., compare the points (20, 7) and (40, 4), or (20, 7) and (70, 3)). For the \textit{If-Then} and the \textit{Dual} models, there are several cases when the run-times for instances with a high $S$ and a low $N$ are higher than for instances with a low $S$ and a high $N$ (e.g., compare the run-time at (70, 3) with that of (20, 7)), although the opposite happens as well (e.g., compare (50, 3) and (40, 4)). 
Thus, it appears that $N$ is the parameter influencing the run-times of \textit{PSums} the most, while for other two models, both $S$ and $N$ are influential, with $S$ having a greater effect. Although these characteristics require additional investigation, one possible reason for such differences in model behaviour could be the relationship between the number of constraints in the models and the problem parameters. 
From Table \ref{tab:modelSummary}, it is known that the number of constraints is mostly determined by the value of $S$ in the \textit{If-Then} and \textit{Dual} models (since $S$ is typically larger than $N$), and by the value of $N$ in the \textit{PSums} model. Combining observations from Table \ref{tab:modelSummary} and Figure \ref{fig:choicePointsByProblemSize}, it appears that the effect of $N$ and $S$ on the run-times is due to their influence on the number of constraints in the models.

Overall, this examination indicates that the superiority of the
\textit{PSums} model without shaving is caused by its stronger propagation
(Table \ref{tab:choicePoints}) and by the fact that propagation is
faster (Figure \ref{fig:choicePointsByProblemSize}).   

When shaving is employed, the \textit{PSums} model also performs better than the \textit{Dual} and \textit{If-Then} models, proving optimality in a greater number of instances (Table \ref{tab:summ_res1}) and finding good-quality solutions faster (Figure \ref{fig:cpModelsHeuristicMRE}). In all models, the shaving procedures make the same number of domain reductions because shaving is based on the $W_q$ and $B_l$ constraints, which are present in all models. However, the time that each shaving iteration takes is different in different models. Our empirical results show that each iteration of shaving takes a smaller amount of time with the \textit{PSums} model than with the \textit{If-Then} and \textit{Dual} models. Thus, with shaving, the \textit{PSums} model performs better than the other two, both because shaving is faster and because subsequent search, if it is necessary, is faster.
%

%
%
%
\subsubsection{Comparison of the \textit{If-Then} and the \textit{Dual} models}
A comparison of the \textit{If-Then} model with \textit{AlternatingSearchAndShaving} with the \textit{Dual} with \textit{AlternatingSearchAndShaving} using
Figure \ref{fig:cpModelsHeuristicMRE} shows that the \textit{If-Then} model is usually able to find good
solutions in a smaller amount of time. 
Moreover, as shown in Table \ref{tab:summ_res2}, the \textit{If-Then} model with \textit{AlternatingSearchAndShaving} finds the
best solution in three more instances, and 
proves optimality in two more instances, than the \textit{Dual} model with the same 
shaving procedure. 
With other shaving procedures and without shaving, the same statistics show almost no difference in the performance of the two models.
It was expected that the \textit{Dual} would outperform the \textit{If-Then} model because
it uses a simpler representation of the balance
equations and expressions for $F$ and $L$, and has a smaller number of if-then
constraints.
(there are 
Table \ref{tab:choicePoints} shows that the \textit{Dual} has to 
explore a smaller number of choice points to find the initial solution. In the 105 instances that 
both of these models solve, the total number of choice points explored by
the \textit{Dual} is also smaller.
However, the \textit{If-Then} model is 
faster, exploring, on average, 895 choice points per second compared to the average of 713 choice points per second 
explored by the \textit{Dual}. 
One possible explanation for the \textit{Dual} being slower is the fact that it has to assign more variables (via propagation) than the other models. In particular, in order to represent a switching policy, the
\textit{Dual} has to assign $S+1$ $w_j$ variables in addition to $N+1$ $k_i$
variables, with $S$ usually being much larger than $N$.
%

\subsection{Performance of the \textit{PSums}-$P_1$ Hybrid}
Experimental results demonstrate that the \textit{PSums}-$P_1$ Hybrid finds
good solutions quickly and is able to prove optimality in a large number of
instances. It should be noted however, that no synergy results from the combination: the number of instances for which optimality is proved does not increase and the run-times do not decrease. Moreover, the hybrid model finds the best-known solution in all test instances simply because there are cases in which only the \textit{PSums} model or only the heuristic is able to do so. No new best solutions are obtained by using the \textit{PSums}-$P_1$ Hybrid to solve the problem. It appears that starting the \textit{PSums} model from the solution found by the heuristic does not lead to a significant increase in the amount of propagation. Also, the fact that the heuristic finds a good-quality solution does not improve the overall performance since, if search has to be used, placing a constraint on $W_q$ that requires all further solutions to be of better quality has little effect on the domains of the decision variables. These observations imply that in order to create a more effective model for the problem, one would need to improve back-propagation by adding new constraints or reformulating the existing ones. If back-propagation is improved, the good-quality heuristic solution may result in better performance of the hybrid approach. 

\subsection{Perspectives}
\label{sec:perspectives}

The CP methods that we have developed are, in some ways, non-standard. More
common approaches when faced with the poor results of our three basic CP
models (without shaving) would have been to create better models, to develop
a global constraint that could represent and efficiently reason about some
relevant sub-structure in the problem, and/or to invent more sophisticated
variable- and value-ordering heuristics. Shaving is more a procedural
technique that must be customized to exploit a particular problem
structure. In contrast, a better model or the creation of a global
constraint are more in-line with the declarative goals of CP. Our decision
to investigate shaving arose from the recognition of the need to more
tightly link the optimization function with the decision variables and the
clear structure of the problem that both appeared and proved to be ideal for
shaving.

We believe that there is scope for better models and novel global
constraints. Modelling problem $P_1$ using CP was not straightforward
because the formulation proposed by Berman et al. contains expressions such
as Equation (\ref{eqn:Fexpression1}), where the upper and lower limits of a
sum of auxiliary variables are decision variables. Such constraints are not
typical for problems usually modelled and solved by CP and there appear to
be no existing global constraints that could have been used to facilitate
our approach. In spite of these issues, our models demonstrate that CP is
flexible enough to support queueing constraints. However, we believe it is
likely that the generalized application of CP to solve a larger class of
queueing control problems will require global constraints specific to
expressions commonly occurring in queueing theory. Given the
back-propagation analysis and the fact that the problem is to find and prove
optimality, we are doubtful that, for $P_1$, sophisticated search heuristics
will perform significantly better than our simple heuristics.

While this is both the first time that CP has been used to solve any
queueing control problem and the first time that instances of $P_1$ have
been provably solved to optimality, the work in this paper can be viewed as
somewhat narrow: it is a demonstration that a particular queueing control
problem can be solved by constraint programming techniques. The work does
not immediately deliver solutions to more general problems, however, we
believe that it does open up a number of directions of inquiry into such
problems.

\begin{enumerate}
\item It appears that there is no standard method within queueing
      theory to address queueing control optimization problems. This
      first application opens the issue of whether CP can become an approach
      of choice for such problems.
\item As noted in Section \ref{sec:introduction}, there is increasing
      interest in incorporating reasoning about uncertainty into CP-based
      problem solving. Queueing theory can provide formulations that allow
      direct calculation of stochastic quantities based on expectation. The
      challenge for CP is to identify the common sub-structures in such
      formulations and develop modelling, inference, and search techniques
      that can exploit them. 
\item Challenging scheduling problems, such as staff rostering at call
      centres \cite{Cezik08a}, consist of queues as well as rich resource
      and temporal constraints (e.g., multiple resource requirements,
      alternative resources of different speeds, task deadlines, precedence
      relations between tasks). We believe that a further
      integration of CP and queueing theory could prove to be a promising
      approach to such problems.
\item The ability to reason about resource allocation under uncertainty is
      an important component of definitions of intelligent behaviour such as
      bounded rationality \cite{Simon97a}. While we cannot claim to have
      made significant contribution in this direction, perhaps ideas from
      queueing theory can serve as the inspiration for such contributions in
      the future.
\end{enumerate}

\section{Related Work and Possible Extensions}
\label{sec:related}
Several papers exist that deal with similar types of problems as the one
considered here. For example, Berman and Larson \citeyear{Berman04} study
the problem of switching workers between two rooms in a retail
facility where the customers in the front room are divided into two
categories, those ``shopping'' in the store and those at the checkout. 
Palmer and Mitrani \citeyear{Palmer04} consider the problem of
switching computational servers between different types of jobs 
where the randomness of user demand may lead to unequal utilization of
resources. Batta, Berman and Wang \citeyear{Batta07a} study the problem of assigning
cross-trained customer service representatives to different types of calls
in a call centre, depending on estimated demand patterns for each type of
call.  These three papers provide examples of problems for which CP could
prove to be a useful approach. Investigating CP solutions to these problems
is therefore one possible direction of future work. 

Further work may also include looking at extensions of the problem discussed in this paper. For example, we may consider a more realistic problem in which there are resource constraints for one or both of the rooms, or in which workers have varying productivity.

Another direction for further work is improvement of the proposed models. In
particular, in all models, but especially in \textit{PSums}, there are
constraints with variable exponents. One idea for improving the performance
of these constraints is to explicitly represent the differences between
switching points (i.e., $k_{i+1}-k_i$) as variables.\footnote{Thanks to an
anonymous reviewer for this suggestion.} Another idea is to
investigate a model based on the logarithms of probabilities rather than the
probabilities themselves. Additionally, ways of increasing the amount of
back-propagation need to be examined.

The goal of this paper was to demonstrate the applicability of constraint programming to solving a particular queueing control problem. The main direction for future work is, therefore, to explore the possibility of further integrating CP and queueing theory in an attempt to address stochastic scheduling and resource allocation problems. Such problems are likely to involve complex constraints, both for encoding the necessary stochastic information and for stating typical scheduling requirements such as task precedences or resource capacities. Combining queueing theory with CP may help in solving such problems.

%
%
\section{Conclusions}
\label{sec:conclusions}
In this paper, a constraint programming approach is proposed for the problem
of finding the optimal states to switch workers between the front room and
the back room of a retail facility under stochastic customer arrival and
service times. This is the first work of which we are aware that examines
solving such stochastic queueing control problems using constraint
programming. The best pure CP method proposed is able to prove optimality in
a large proportion of instances within a 10-minute time limit. Previously,
there existed no non-heuristic solution to this problem aside from naive
enumeration. As a result of our experiments, we hybridized the best pure CP model with the heuristic proposed for this problem in the literature. This hybrid technique is able to achieve performance that is equivalent to, or better than, that of each of the individual approaches alone: it is able to find very good solutions in a negligible amount of time due to the use of the heuristic, and is able to prove optimality in a large proportion of problem instances within 10 CPU-minutes due to the CP model. 

This paper demonstrates that constraint programming can be a good approach for solving a 
queueing control optimization problem. For queueing problems for which
optimality is important or heuristics do not perform well, CP may prove to
be an effective methodology. 

\appendix
\section{Constraint Derivations}
In this section, the derivations of constraints in the \textit{PSums} model and expressions for the auxiliary variables and constraints are 
presented. 

\subsection{Closed-form Expressions in the \textit{PSums} model}
The \textit{PSums} model has two sets of probability variables, $P(k_i)$, $i = 0, 1, \dots, N$, the probability of there being $k_i$ customers in the front room, and $PSums(k_i)$, $i = 0, 1, \dots, N-1$, the sum of probabilities between two switching point variables. 
Balance equations are not explicitly stated in this model. However, expressions for $P(k_i)$ and $PSums(k_i)$ are derived in such a way that the
balance equations are satisfied. The technique used for these derivations is similar
to that used by Berman et al. \citeyear{Berman05a} to simplify the calculation of probabilities.

Consider the balance equation $P(j)\lambda = P(j+1)i\mu$, which is true for $j = k_{i-1}, k_{i-1}+1, \dots, k_i-1$ and any $i \in \{1, 2 \dots, N\}$. In particular, this subset of the balance equations is 
\begin{eqnarray} \nonumber
P(k_{i-1}) \lambda &=& P(k_{i-1}+1) i \mu \\ \nonumber
P(k_{i-1}+1) \lambda &=& P(k_{i-1}+2) i \mu \\ \nonumber
{} & \vdots & {} \\ \nonumber
P(k_i-1) \lambda &=& P(k_{i}) i \mu. \nonumber
\end{eqnarray}

\noindent These equations imply the following expressions:

\begin{eqnarray} \label{eqn:balEqns}
\frac{P(k_{i-1})\lambda}{i\mu} &=& P(k_{i-1}+1)\\ \nonumber
\frac{P(k_{i-1}+1)\lambda}{i\mu} &=& P(k_{i-1}+2)\\ \nonumber
{} & \vdots & {} \\ \nonumber
\frac{P(k_i-1)\lambda}{i\mu} &=& P(k_{i}). \nonumber
\end{eqnarray}

Combining these together, we get $P(k_i) = {\left(\lambdamuiRatio\right)}^{k_i - k_{i-1}}P(k_{i-1})$ for all $i$ from 1 to $N$, or 
\begin{eqnarray}
P(k_{i+1}) &=& {\left(\ratio\right)}^{k_{i+1}-k_i}P(k_i), \;\;\; \forall i \in \{0, 1, \dots, N-1\}. \label{eqn:switchProb2}
\end{eqnarray}
This equation is included in the \textit{PSums} model and has previously been stated as Equation (\ref{eqn:switchProb}).

Similarly, from Equation (\ref{eqn:balEqns}), we see that $P(k_{i-1}+1) = \lambdamuiRatio P(k_{i-1})$ for all $i$ from 1 to $N$, or  
\begin{eqnarray} \label{eqn:P(ki+1)}
P(k_i+1) = \ratio P(k_i), \;\;\; \forall i \in \{0, 1, \dots, N\}.
\end{eqnarray}   

Using Equation (\ref{eqn:P(ki+1)}), an expression for $PSums(k_i)$, the sum of probabilities $P(j)$ for $j$ between $k_i$ and $k_{i+1}-1$, can be derived as follows: 
\begin{eqnarray} \nonumber
PSums(k_i) &=& \sum_{j=k_i}^{k_{i+1}-1} P(j) \\ \nonumber
{} &=& P(k_i)+P(k_i+1)+P(k_i+2)+ \dots +P(k_{i+1}-1) \\ \nonumber
{} &=& P(k_i)+P(k_i)\ratio+P(k_i){\left(\ratio\right)}^2 \\ \nonumber
{} &{}& + \dots + P(k_i){\left(\ratio\right)}^{k_{i+1}-k_i-1} \\ \nonumber
{} &=& P(k_i)\left[1+\ratio+{\left(\ratio \right)}^2 + \dots + {\left(\ratio\right)}^{k_{i+1}-k_i-1} \right] \\ \nonumber
\\ \nonumber
\\ 
{} &=& \left\{ 
\begin{array}{rr}
P(k_i) \frac{\displaystyle1-{\left[\displaystyle\ratio\right]}^{k_{i+1}-k_i}}{\displaystyle1-\displaystyle\ratio} & \mbox{if } \ratio \neq 1\\
\\
P(k_i) (\diff)  & otherwise. 
\end{array} \right. \label{eqn:pSums2} 	
\end{eqnarray}
\noindent The last step in the derivation is based on the observation that the expression $[1+\ratio+{\left(\ratio \right)}^2 + \dots + {\left(\ratio\right)}^{k_{i+1}-k_i-1}]$ is a geometric series with the common ratio $\ratio$. When $\ratio$ is 1, the expression is simply a sum of $k_{i+1}-k_i$ ones. The expression for $PSums(k_i)$ has been previously stated as Equation (\ref{eqn:pSums}). 

\subsubsection{Expected Number of Workers Constraint}
$F$ can be expressed in terms of $P(k_i)$ and $PSums(k_i)$ using the following sequence of steps:
\begin{eqnarray}
F &=& \sum_{i=1}^{N}{\sum_{j=k_{i-1}+1}^{k_{i}} iP(j) } \nonumber \\
{} &=& \sum_{i=1}^{N}{i\left[P(k_{i-1}+1)+P(k_{i-1}+2)+ \dots + P(k_i-1) + P(k_i) \right]} \nonumber \\
{} &=& \sum_{i=1}^{N}{i\left[PSums(k_{i-1})-P(k_{i-1})+P(k_i) \right]}. \label{eqn:PSumsFexpression2}
\end{eqnarray}
\subsubsection{Expected Number of Customers Constraints}
The equation for $L$ can be derived in a similar manner:
\begin{eqnarray}
L &=& \sum_{j=k_0}^{k_N}jP(j) \nonumber \\
{} &=& \sum_{j=k_0}^{k_1-1}jP(j) + \sum_{j=k_1}^{k_2-1}jP(j) + \dots + \sum_{j=k_{N-1}}^{k_N-1}jP(j) + k_NP(k_N) \nonumber \\
{} &=& L(k_0)+L(k_1)+ \dots + L(k_{N-1})+ k_NP(k_N) \nonumber \\
{} &=& \sum_{i = 0}^{N-1} L(k_i) + k_NP(k_N) \label{eqn:PSumsLexpression2}
\end{eqnarray}
where\label{Lki2}
\begin{eqnarray} \nonumber
L(k_i) &=& k_iP(k_i)+(k_i+1)P(k_i+1)+(k_i+2)P(k_i+2) + \dots + (k_{i+1}-1)P(k_{i+1}-1) \\ \nonumber
{} &=& k_iP(k_i)+ k_iP(k_i+1)+k_iP(k_i+2) + \dots + k_iP(k_{i+1}-1)+P(k_i+1) \\ \nonumber
{} &+& 2P(k_i+2) + \dots + (k_{i+1}-k_i-1)P(k_{i+1}-1) \\ \nonumber 
{} &=& k_i[P(k_i)+P(k_i+1)+P(k_i+2)+ \dots + P(k_{i+1}-1)]+ P(k_i+1) \\ \nonumber
{} &+& 2P(k_i+2)+ \dots + (k_{i+1}-k_i-1)P(k_{i+1}-1) \\ \nonumber
{} &=& k_iPSums(k_i)+P(k_i)\ratio + 2P(k_i){\left(\ratio\right)}^2 \\ \nonumber
{} &+& \dots + (k_{i+1}-k_i+1)P(k_i){\left(\ratio\right)}^{k_{i+1}-k_i-1} \\ \nonumber
{} &=& k_iPSums(k_i) + P(k_i)\ratio \\ \nonumber
{} & \times & \left[1+2\ratio+3{\left(\ratio\right)}^2+ \dots \right. \\ \nonumber
{} & + & \left. (k_{i+1}-k_i-1)P(k_i){\left(\ratio\right)}^{k_{i+1}-k_i-2} \right] 
\end{eqnarray}
\begin{eqnarray}
{} &=& k_iPSums(k_i)+P(k_i)\ratio\sum_{n=0}^{k_{i+1}-k_i-1}{n{\left(\ratio\right)}^{n-1}}\\ \nonumber
{} &=& k_iPSums(k_i)+P(k_i)\ratio \\ \nonumber 
{} & {} & \times \left[\frac{ {\left(\ratio\right)}^{k_{i+1}-k_i-1}(k_i-k_{i+1}) + {\left(\ratio\right)}^{k_{i+1}-k_i}(k_{i+1}-k_i-1)+1 }{{\left(1-\ratio\right)}^2} \right]. \nonumber
\end{eqnarray}

\subsection{Auxiliary Variables}
All constraint programming models contain Equation (\ref{eqn:betaSumEquation}) (restated below as Equation (\ref{eqn:betaSumEquation2})) for defining the $\beta Sum(k_i)$ variables, which are necessary for expressing an auxiliary constraint that ensures that the balance equations have a unique solution. The validity of this equation is proved by the following derivation, which uses the formula for the sum of a finite geometric series in the last step:
%
\begin{eqnarray}\label{betaSum2}
\beta Sum(k_i) &=& \sum_{j=k_{i-1}+1}^{k_i} \beta_j \nonumber \\
{} &=& {\left({\ratioSimple}\right)}^{k_{i-1}+1-k_0}{\left({\iRatio}\right)}^{k_{i-1}+1-k_{i-1}}X_i \nonumber \\
{} &+& {\left({\ratioSimple}\right)}^{k_{i-1}+2-k_0}{\left({\iRatio}\right)}^{k_{i-1}+2-k_{i-1}}X_i + \dots +
{\left({\ratioSimple}\right)}^{k_{i}-k_0}{\left({\iRatio}\right)}^{k_{i}-k_{i-1}}X_i \nonumber \\
{} &=& X_i{\left(\ratioSimple\right)}^{k_{i-1}-k_0+1}{\left({\iRatio}\right)} \nonumber \\
{} &\times& \left[1+\ratioSimple\iRatio+{\left(\ratioSimple\right)}^{2}{\left(\iRatio\right)}^2 +  \dots + {\left(\ratioSimple\right)}^{k_1-k_0-(k_{i-1}-k_0+1)}{\left(\iRatio\right)}^{k_i-k_{i-1}-1}\right] \nonumber \\
{} &=& X_i{\left(\ratioSimple\right)}^{k_{i-1}-k_0+1}\left(\iRatio\right) \sum_{n=0}^{k_i-k_{i-1}-1}{\left(\lambdamuiRatio\right)}^n \nonumber \\
\nonumber \\
\nonumber \\
{} &=& \left\{ 
\begin{array}{rr}
X_i{\left(\displaystyle\ratioSimple\right)}^{k_{i-1}-k_0+1}\left(\displaystyle\iRatio\right) \left[\frac{1-{\left(\displaystyle\lambdamuiRatio\right)}^{k_i-k_{i-1}}}{1-\left(\displaystyle\lambdamuiRatio\right)} \right] & \mbox{if } \lambdamuiRatio \neq 1\\
\\
X_i{\left(\displaystyle\ratioSimple\right)}^{k_{i-1}-k_0+1}\left(\displaystyle\iRatio\right) (k_i-k_{i-1}) & otherwise. 
\end{array} \right. \label{eqn:betaSumEquation2}	
\end{eqnarray}
\\
\paragraph{Acknowledgments} This research was supported in part by the
Natural Sciences and Engineering Research Council and ILOG, S.A.  Thanks to
Nic Wilson and Ken Brown for discussions and comments on this work, and to Tom Goguen for careful proofreading of the final copy. A
preliminary version of parts of this work has been previously published
\cite{Terekhov07a}.
%
%
%
\bibliographystyle{theapa}

\begin{thebibliography}{}

\bibitem[\protect\BCAY{Baptiste, Le~Pape,\ \BBA\ Nuijten}{Baptiste
  et~al.}{2001}]{Baptiste01a}
Baptiste, P., Le~Pape, C., \BBA\ Nuijten, W. \BBOP2001\BBCP.
\newblock {\Bem Constraint-based Scheduling}.
\newblock Kluwer Academic Publishers.

\bibitem[\protect\BCAY{Batta, Berman,\ \BBA\ Wang}{Batta
  et~al.}{2007}]{Batta07a}
Batta, R., Berman, O., \BBA\ Wang, Q. \BBOP2007\BBCP.
\newblock \BBOQ Balancing staffing and switching costs in a service center with
  flexible servers\BBCQ\
\newblock {\Bem European Journal of Operational Research}, {\Bem 177},
  924--938.

\bibitem[\protect\BCAY{Beck\ \BBA\ Prestwich}{Beck\ \BBA\
  Prestwich}{2004}]{Beck04f}
Beck, J.~C.\BBACOMMA\  \BBA\ Prestwich, S. \BBOP2004\BBCP.
\newblock \BBOQ Exploiting dominance in three symmetric problems\BBCQ\
\newblock In {\Bem Fourth International Workshop on Symmetry and Constraint
  Satisfaction Problems}.

\bibitem[\protect\BCAY{Beck\ \BBA\ Wilson}{Beck\ \BBA\ Wilson}{2007}]{Beck07a}
Beck, J.~C.\BBACOMMA\  \BBA\ Wilson, N. \BBOP2007\BBCP.
\newblock \BBOQ Proactive algorithms for job shop schedulng with probabilistic
  durations\BBCQ\
\newblock {\Bem Journal of Artificial Intelligence Research}, {\Bem 28},
  183--232.

\bibitem[\protect\BCAY{Berman\ \BBA\ Larson}{Berman\ \BBA\
  Larson}{2004}]{Berman04}
Berman, O.\BBACOMMA\  \BBA\ Larson, R. \BBOP2004\BBCP.
\newblock \BBOQ A queueing control model for retail services having back room
  operations and cross-trained workers\BBCQ\
\newblock {\Bem Computers and Operations Research}, {\Bem 31\/}(2), 201--222.

\bibitem[\protect\BCAY{Berman, Wang,\ \BBA\ Sapna}{Berman
  et~al.}{2005}]{Berman05a}
Berman, O., Wang, J., \BBA\ Sapna, K.~P. \BBOP2005\BBCP.
\newblock \BBOQ Optimal management of cross-trained workers in services with
  negligible switching costs\BBCQ\
\newblock {\Bem European Journal of Operational Research}, {\Bem 167\/}(2),
  349--369.

\bibitem[\protect\BCAY{Brown\ \BBA\ Miguel}{Brown\ \BBA\
  Miguel}{2006}]{Brown06}
Brown, K.~N.\BBACOMMA\  \BBA\ Miguel, I. \BBOP2006\BBCP.
\newblock \BBOQ Uncertainty and change\BBCQ\
\newblock In Rossi, F., van Beek, P., \BBA\ Walsh, T.\BEDS, {\Bem Handbook of
  Constraint Programming}, \BCH~21, \BPGS\ 731--760. Elsevier.

\bibitem[\protect\BCAY{Caseau\ \BBA\ Laburthe}{Caseau\ \BBA\
  Laburthe}{1996}]{Caseau96}
Caseau, Y.\BBACOMMA\  \BBA\ Laburthe, F. \BBOP1996\BBCP.
\newblock \BBOQ Cumulative scheduling with task intervals\BBCQ\
\newblock In {\Bem Proceedings of the Joint International Conference and
  Symposium on Logic Programming}, \BPGS\ 363--377. MIT Press.

\bibitem[\protect\BCAY{Cezik\ \BBA\ L'Ecuyer}{Cezik\ \BBA\
  L'Ecuyer}{2008}]{Cezik08a}
Cezik, M.~T.\BBACOMMA\  \BBA\ L'Ecuyer, P. \BBOP2008\BBCP.
\newblock \BBOQ Staffing multiskill call centers via linear programming and
  simulation\BBCQ\
\newblock {\Bem Management Science}, {\Bem 54\/}(2), 310--323.

\bibitem[\protect\BCAY{Demassey, Artigues,\ \BBA\ Michelon}{Demassey
  et~al.}{2005}]{Demassey05}
Demassey, S., Artigues, C., \BBA\ Michelon, P. \BBOP2005\BBCP.
\newblock \BBOQ Constraint-propagation-based cutting planes: An application to
  the resource-constrained project scheduling problem\BBCQ\
\newblock {\Bem INFORMS Journal on Computing}, {\Bem 17\/}(1), 52--65.

\bibitem[\protect\BCAY{Fox}{Fox}{1983}]{Fox83}
Fox, M.~S. \BBOP1983\BBCP.
\newblock {\Bem Constraint-Directed Search: A Case Study of Job-Shop
  Scheduling}.
\newblock Ph.D.\ thesis, Carnegie Mellon University, Intelligent Systems
  Laboratory, The Robotics Institute, Pittsburgh, PA.
\newblock CMU-RI-TR-85-7.

\bibitem[\protect\BCAY{Gross\ \BBA\ Harris}{Gross\ \BBA\
  Harris}{1998}]{Gross98a}
Gross, D.\BBACOMMA\  \BBA\ Harris, C. \BBOP1998\BBCP.
\newblock {\Bem Fundamentals of Queueing Theory}.
\newblock John Wiley \& Sons, Inc.

\bibitem[\protect\BCAY{Hnich, Smith,\ \BBA\ Walsh}{Hnich
  et~al.}{2004}]{Hnich04}
Hnich, B., Smith, B.~M., \BBA\ Walsh, T. \BBOP2004\BBCP.
\newblock \BBOQ Dual modelling of permutation and injection problems\BBCQ\
\newblock {\Bem Journal of Artificial Intelligence Research}, {\Bem 21},
  357--391.

\bibitem[\protect\BCAY{Martin\ \BBA\ Shmoys}{Martin\ \BBA\
  Shmoys}{1996}]{Martin96}
Martin, P.\BBACOMMA\  \BBA\ Shmoys, D.~B. \BBOP1996\BBCP.
\newblock \BBOQ A new approach to computing optimal schedules for the job shop
  scheduling problem\BBCQ\
\newblock In {\Bem Proceedings of the Fifth Conference on Integer Programming
  and Combinatorial Optimization}, \BPGS\ 389--403.

\bibitem[\protect\BCAY{Palmer\ \BBA\ Mitrani}{Palmer\ \BBA\
  Mitrani}{2004}]{Palmer04}
Palmer, J.\BBACOMMA\  \BBA\ Mitrani, I. \BBOP2004\BBCP.
\newblock \BBOQ Optimal server allocation in reconfigurable clusters with
  multiple job types\BBCQ\
\newblock In {\Bem Proceedings of the International Conference on Computational
  Science and Its Applications (ICCSA'04)}, \BPGS\ 76--86.

\bibitem[\protect\BCAY{Policella, Smith, Cesta,\ \BBA\ Oddi}{Policella
  et~al.}{2004}]{Policella04a}
Policella, N., Smith, S.~F., Cesta, A., \BBA\ Oddi, A. \BBOP2004\BBCP.
\newblock \BBOQ Generating robust schedules through temporal flexibility\BBCQ\
\newblock In {\Bem Proceedings of the Fourteenth International Conference on
  Automated Planning and Scheduling (ICAPS'04)}, \BPGS\ 209--218.

\bibitem[\protect\BCAY{Simon}{Simon}{1997}]{Simon97a}
Simon, H.~A. \BBOP1997\BBCP.
\newblock {\Bem Models of Bounded Rationality}, \lowercase{\BVOL}~3.
\newblock MIT Press.

\bibitem[\protect\BCAY{Smith}{Smith}{2005}]{Smith05a}
Smith, B.~M. \BBOP2005\BBCP.
\newblock \BBOQ Modelling for constraint programming\BBCQ\
\newblock Lecture Notes for the First International Summer School on Constraint
  Programming.
\newblock Available at: http://www.math.unipd.it/\~{}frossi/cp-school/.

\bibitem[\protect\BCAY{Smith}{Smith}{2006}]{Smith06}
Smith, B.~M. \BBOP2006\BBCP.
\newblock \BBOQ Modelling\BBCQ\
\newblock In Rossi, F., van Beek, P., \BBA\ Walsh, T.\BEDS, {\Bem Handbook of
  Constraint Programming}, \BCH~11, \BPGS\ 377--406. Elsevier.

\bibitem[\protect\BCAY{Solver}{Solver}{2006}]{Solver62}
Solver \BBOP2006\BBCP.
\newblock {\Bem {ILOG} Scheduler 6.2 User's Manual and Reference Manual}.
\newblock ILOG, S.A.

\bibitem[\protect\BCAY{Sutton, Howe,\ \BBA\ Whitley}{Sutton
  et~al.}{2007}]{Sutton07a}
Sutton, A.~M., Howe, A.~E., \BBA\ Whitley, L.~D. \BBOP2007\BBCP.
\newblock \BBOQ Using adaptive priority weighting to direct search in
  probabilistic scheduling\BBCQ\
\newblock In {\Bem Proceedings of the Seventeenth International Conference on
  Automated Planning and Scheduling}, \BPGS\ 320--327.

\bibitem[\protect\BCAY{Tadj\ \BBA\ Choudhury}{Tadj\ \BBA\
  Choudhury}{2005}]{Tadj05a}
Tadj, L.\BBACOMMA\  \BBA\ Choudhury, G. \BBOP2005\BBCP.
\newblock \BBOQ Optimal design and control of queues\BBCQ\
\newblock {\Bem {TOP}}, {\Bem 13\/}(2), 359--412.

\bibitem[\protect\BCAY{Tarim, Manandhar,\ \BBA\ Walsh}{Tarim
  et~al.}{2006}]{Tarim06a}
Tarim, S.~A., Manandhar, S., \BBA\ Walsh, T. \BBOP2006\BBCP.
\newblock \BBOQ Stochastic constraint programming: A scenario-based
  approach\BBCQ\
\newblock {\Bem Constraints}, {\Bem 11\/}(1), 53--80.

\bibitem[\protect\BCAY{Tarim\ \BBA\ Miguel}{Tarim\ \BBA\
  Miguel}{2005}]{Tarim05b}
Tarim, S.~A.\BBACOMMA\  \BBA\ Miguel, I. \BBOP2005\BBCP.
\newblock \BBOQ A hybrid {B}enders' decomposition method for solving stochastic
  constraint programs with linear recourse.\BBCQ\
\newblock In {\Bem Joint ERCIM/CoLogNET International Workshop on Constraint
  Solving and Constraint Logic Programming}, \BPGS\ 133--148.

\bibitem[\protect\BCAY{Terekhov}{Terekhov}{2007}]{Terekhov07c}
Terekhov, D. \BBOP2007\BBCP.
\newblock \BBOQ Solving queueing design and control problems with constraint
  programming\BBCQ\
\newblock Master's thesis, Department of Mechanical and Industrial Engineering,
  University of Toronto.

\bibitem[\protect\BCAY{Terekhov\ \BBA\ Beck}{Terekhov\ \BBA\
  Beck}{2007}]{Terekhov07a}
Terekhov, D.\BBACOMMA\  \BBA\ Beck, J.~C. \BBOP2007\BBCP.
\newblock \BBOQ Solving a stochastic queueing control problem with constraint
  programming\BBCQ\
\newblock In {\Bem Proceedings of the Fourth International Conference on
  Integration of {AI} and {OR} Techniques in Constraint Programming for
  Combinatorial Optimization Problems (CPAIOR'07)}, \BPGS\ 303--317.
  Springer-Verlag.

\bibitem[\protect\BCAY{van Dongen}{van Dongen}{2006}]{vanDongen06}
van Dongen, M. R.~C. \BBOP2006\BBCP.
\newblock \BBOQ Beyond singleton arc consistency\BBCQ\
\newblock In {\Bem Proceedings of the Seventeenth European Conference on
  Artificial Intelligence (ECAI'06)}, \BPGS\ 163--167.

\bibitem[\protect\BCAY{Walsh}{Walsh}{2002}]{Walsh02}
Walsh, T. \BBOP2002\BBCP.
\newblock \BBOQ Stochastic constraint programming\BBCQ\
\newblock In {\Bem Proceedings of the Fifteenth European Conference on
  Artificial Intelligence}, \BPGS\ 111--115.

\end{thebibliography}

%
\end{document}